%% file: main.tex
\documentclass[accepted]{uai2026} 
                        
\usepackage[american]{babel}
\usepackage[table]{xcolor} 
\usepackage{natbib} 
\usepackage{mathtools} 

\usepackage{booktabs} 
\usepackage{tikz} 
\usepackage{defs}
\usepackage{booktabs}
\usepackage{arydshln}
\usepackage{times}  
\usepackage{helvet}  
\usepackage{courier}  
\usepackage{graphicx} 
\usepackage{caption}
\usepackage{enumitem}
\setlist{nosep}
\usepackage{multirow}
\usepackage{subcaption}

\definecolor{improvegreen}{RGB}{230, 245, 224}  
\definecolor{degradered}{RGB}{255, 230, 230}    
\definecolor{mutedgreen}{RGB}{76, 153, 0}
\definecolor{mutedyellow}{RGB}{204, 153, 0}
\definecolor{mutedred}{RGB}{204, 76, 76}

\title{Ensemble Diversity Optimization for Subjective Supervision
}
\author[1]{\href{mailto:<x.cui@mmu.ac.uk>?Subject=Your UAI 2026 paper on EDO}{Xia Cui}}
\author[2]{Ziyi Huang}
\author[1]{N. R. Abeynayake}
\affil[1]{%
    School of Computing and Mathematics, Manchester Metropolitan University, Manchester, UK.
}
\affil[2]{%
   School of Computer Science, Hubei University, Wuhan, China.
}

\begin{document}
 
\maketitle

\begin{abstract}
Subjective NLP tasks often exhibit systematic annotator disagreement, requiring models that represent uncertainty rather than collapse it. We introduce Ensemble Diversity Optimization (EDO), a prediction-space framework that jointly optimizes ensemble weights, effective cardinality, and calibration through a unified differentiable objective. EDO learns ensemble composition and size end-to-end via Gumbel–Softmax relaxation and incorporates a signed diversity regularizer, tuned on validation data, to steer optimization toward either preserving or suppressing disagreement. This regularization prevents ensemble collapse and enables controlled navigation of the utility–calibration trade-off. The framework integrates a soft F1 surrogate, class-weighted cross-entropy to address imbalance, and reliability-weighted diversity to regulate intra-ensemble variability.
Experiments on four subjective text-classification benchmarks (ArMIS, ConvAbuse, HS-Brexit, MD-Agreement) show that EDO substantially improves probabilistic calibration, reducing cross-entropy (40–78\% depending on baseline) and lowering Brier scores relative to Soft-CE, Soft-MD, Top-5 Voting, and WEL, while maintaining competitive F1 and better alignment with annotator distributions. These results demonstrate that jointly optimizing ensemble structure with a signed diversity regularizer provides an efficient, model-agnostic approach for modeling human subjectivity in supervised learning.
\end{abstract}

\input{scripts/01_introduction}

\input{scripts/02_relatedwork}
\input{scripts/03_methods}

\input{scripts/04_experiments}
\input{scripts/05_results}
\input{scripts/06_limitations}
\input{scripts/07_conclusions}
\input{scripts/E_acknowledge}

\bibliography{datainfo}

\input{scripts/supplement}

\end{document}

%% file: scripts/01_introduction.tex
\section{Introduction}

Many NLP tasks exhibit substantial and systematic annotator disagreement. In domains such as content moderation, hate speech detection, and sentiment analysis, divergent annotations arise from semantic ambiguity, contextual dependence, or variation in annotator expertise rather than annotation error~\citep{Snow:EMNLP:2008,Plank:EMNLP:2022,Uma:JAIR:2022,Cui:RANLP:2025}. These settings challenge standard supervised learning assumptions, as the target is not a single latent label but a distribution over plausible human interpretations. Nevertheless, prevailing practice aggregates annotations into a single target, discarding distributional information and inducing overfitting to dominant interpretations~\citep{Davani:TACL:2022,Liu:ACL:2022}.

\begin{figure}[tbh]
\centering
\includegraphics[width=0.45\textwidth]{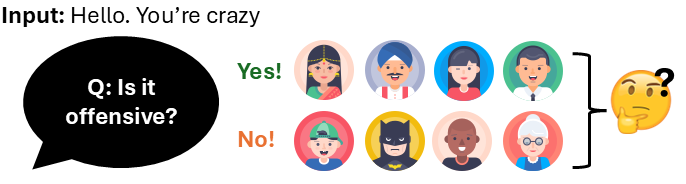}
\caption{Illustration of annotator disagreement.}
\label{fig:annotator-dis}
\end{figure}

Soft-label supervision preserves annotator label distributions and provides richer training signals~\citep{Uma:AAAI:2020,Rizzi:COLING:2024}, but continues to optimize a single predictive model. This setting aligns with partial-label learning (PLL)~\citep{Cour:JMLR:2011}, where each instance is associated with a set of candidate labels. Unlike classical PLL, which assumes a single hidden ground truth, subjective tasks often involve genuine multiplicity, where multiple labels are simultaneously valid. As a result, existing approaches provide no explicit mechanism for regulating internal predictive variability or distinguishing systematic subjectivity from annotation noise, particularly under class imbalance or heterogeneous annotator reliability.

Ensemble methods offer a principled mechanism for representing multiple plausible hypotheses and have been widely used to capture predictive uncertainty~\citep{Fort:2019:deepensembles}. However, existing approaches typically rely on fixed architectures and treat diversity as an emergent property rather than an explicit optimization objective. This contrasts with the unified theory of ensemble diversity~\citep{Wood:JMLR:2023}, which proves that for losses like cross-entropy, ensemble error decomposes exactly into bias, variance and diversity terms. This establishes diversity as a fundamental component of generalization, not an auxiliary heuristic. Despite this, practical methods lack mechanisms to optimize diversity in alignment with human subjectivity.

We propose \textbf{Ensemble Diversity Optimization (EDO)}, a prediction‑space framework for learning under subjective and imbalanced supervision. EDO introduces three key operational components: (i) a \emph{signed}, reliability-weighted diversity regularizer that either encourages disagreement (to preserve epistemic uncertainty) or suppresses it when the observed divergence is driven by structural factors such as imbalance or sparse annotator coverage; (ii) a unified multi-objective optimization procedure jointly balancing predictive utility (micro-F1), calibration (class-weighted cross-entropy), and internal diversity; and (iii) differentiable learning of ensemble structure, including cardinality, via Gumbel–Softmax relaxation~\citep{Jang:2017:gumbelsoftmax}. Crucially, the diversity coefficient and direction ($s \in \{-1, +1\}$) are treated as fixed, validation-tuned hyperparameters that steer the joint optimization trajectory, preventing ensemble collapse and enabling the reliability-weighting and size-selection components to operate synergistically. While we freeze backbone parameters in our experiments to isolate uncertainty attributable specifically to annotator disagreement and control computational cost, EDO operates entirely in prediction space and is fully model-agnostic, natively supporting end-to-end fine-tuning when task-specific adaptation is desired.

Experiments on four datasets from the LeWiDi benchmark~\citep{LeWiDi:SemEval:2023} show that EDO's joint optimization substantially improves probabilistic calibration (reducing cross-entropy by over 40\% and lowering Brier score relative to five baselines) while maintaining competitive F1 and stronger alignment with annotator distributions. These results demonstrate that jointly optimizing ensemble structure with a directional diversity regularizer provides an efficient and principled approach for modeling human subjectivity in supervised learning.

The code and data are available on GitHub\footnote{\url{https://github.com/MMUNLP/EDO}}.

%% file: scripts/02_relatedwork.tex
\section{Related Work}

\paragraph{Partial-label learning.}
Partial-label learning (PLL) considers instances annotated with a candidate set of labels, exactly one of which is assumed correct~\citep{Cour:JMLR:2011}. While EDO similarly treats annotations as sets rather than singletons, it departs from PLL's core assumption of a hidden ground truth. In subjective tasks, disagreement often reflects irreducible epistemic uncertainty, where multiple interpretations may be simultaneously valid. EDO therefore extends PLL to settings of genuine multiplicity, aiming to represent the full distribution of human judgments rather than recover a latent singleton.

\paragraph{Subjective supervision and soft labels.}
Annotator disagreement often reflects genuine interpretive variation rather than noise~\citep{Aroyo:IEEE:2015,Pavlick:NAACL:2019,Uma:JAIR:2022}, motivating methods that preserve rather than collapse this variability. Soft-label supervision retains the empirical distribution of annotator responses~\citep{Davani:TACL:2022} and improves robustness and calibration~\citep{Swayamdipta:EMNLP:2020,Uma:AAAI:2020,Zheng:EMNLP:2021,Rizzi:COLING:2024}. Existing approaches align predictions with annotator distributions via KL-based objectives, entropy regularization, or disagreement-aware metrics such as Manhattan Distance~\citep{Rizzi:COLING:2024,Tian:NPL:2024}, while entropy-based reweighting highlights ambiguous instances~\citep{Zhang:ACL:2020,Liu:TNLS:2021}. However, they still optimize a single predictive model and lack mechanisms for explicitly controlling intra-model or intra-ensemble variability.

\paragraph{Annotator modeling and expected-loss approaches.}
Annotator-specific models~\citep{Akhtar:HS-Brexit:2021,Gordon:EMNLP:2021,Xu:ICNLSP:2024} capture individual biases when annotator metadata is available, and expected-loss formulations approximate the latent distribution of annotator responses~\citep{Uma:AAAI:2020,Leonardelli:EMNLP:MD-Agreement:2021}. These methods are effective in settings with reliable annotator information. They depend on either explicit metadata or fixed aggregation schemes and generally lack mechanisms to present structured epistemic uncertainty within the predictive model.

\paragraph{Ensembles and diversity.}
Ensembles provide a natural mechanism for capturing multiple plausible interpretations and have been used to model annotator variability~\citep{Leonardelli:EMNLP:MD-Agreement:2021,Fleisig:EMNLP:2023}. However, most ensemble architectures are fixed, and diversity is treated as an emergent property rather than an optimization objective. Early work on negative correlation learning~\citep{Liu:1999:NLL} introduced diversity as an explicit regularizer to encourage specialization among ensemble members, while determinantal point processes~\citep{Macchi:1975:DPP,Launay:2020:DPP} provide a probabilistic framework for modeling subset diversity. More recent differentiable approaches include sparsely-gated mixture-of-experts~\citep{Shazeer:ICLR:2017}, which learns routing weights for expert selection, and Gumbel-based methods for $k$-subset sampling~\citep{Chaudhuri:ICLR:2019,Wijk:ICLR:2025} that enable gradient-based subset selection. 

EDO builds on these foundations but differs in three key respects: (i) it integrates diversity control directly into a joint multi-objective loss balancing utility, calibration, and disagreement; (ii) it treats the diversity coefficient and direction as validation-tuned hyperparameters that steer optimization along the utility--calibration Pareto frontier, rather than learning them end-to-end; and (iii) it operates entirely in prediction space, making it model-agnostic and computationally efficient. Weak Ensemble Learning (WEL)~\citep{Huang:WEL:2025} moves toward more principled ensemble optimization by pairing Random Select and Per-Annotator supervision with objectives for utility and calibration. Nonetheless, it lacks explicit control over pairwise disagreement, allowing diversity to collapse even in the presence of substantial annotator disagreement.

\paragraph{Theoretical grounding and multi-objective learning.}
Theoretical analyses decompose ensemble error into bias, variance and diversity components~\citep{Wood:JMLR:2023}, identifying diversity as a principled contributor to generalization. However, these results do not directly yield practical optimization strategies for subjective or imbalanced settings. Multi-objective optimization has been applied to balance competing goals such as predictive utility, calibration, and fairness~\citep{Beutel:NeurIPS:2019}, and diversity-promoting regularizers can prevent convergence to a single dominant interpretation~\citep{Li:NACCL:2016,Zhang:ACLSD:2020}. Still, these methods are typically applied to fixed ensembles and do not jointly optimize diversity, calibration, and utility within a unified differentiable framework.

\paragraph{Our contribution.}
EDO addresses these limitations by integrating stochastic annotator supervision with explicit, reliability-aware optimization of ensemble diversity. It jointly optimizes predictive utility, calibration, and internal disagreement while learning ensemble structure end-to-end via Gumbel--Softmax relaxation~\citep{Jang:2017:gumbelsoftmax}, providing a practical and flexible framework for modeling subjective uncertainty without requiring annotator metadata or hand-designed ensemble architectures. Crucially, reliability weights emerge implicitly from gradient-based balancing of calibration and utility objectives, rather than being precomputed uncertainty scores; the signed diversity regularizer then steers this joint optimization to either preserve epistemic variability ($s=-1$) or suppress structural noise ($s=+1$).

%% file: scripts/03_methods.tex
\section{Methodology}
\label{sec:method}
We propose \textbf{Ensemble Diversity Optimization} (EDO), a gradient-based framework that treats ensemble disagreement as a signal of annotator subjectivity rather than noise. EDO optimizes reliability-aware ensemble weights $\mathbf{w}$ and effective size $K$ via a signed diversity objective: it preserves disagreement reflecting genuine subjectivity while suppressing variation‑induced or imbalance‑driven divergence. Operating solely in prediction space (with frozen backbones), EDO isolates annotation-driven uncertainty from training parameters, yielding efficient, model-agnostic control over intra-ensemble variability.
EDO extends Weak Ensemble Learning (WEL)~\citep{Huang:WEL:2025} through three innovations: (i) a signed, reliability-weighted diversity objective; (ii) class-weighted calibration for imbalanced soft labels; and (iii) differentiable ensemble-structure learning, replacing WEL's derivative-free optimization.

\subsection{Problem Setup}

We consider a dataset $\mathcal{D} = \{(x_i, \{y_i^{(a)}\}_{a=1}^{A_i})\}_{i=1}^N$, where each instance $x_i$ is annotated by $A_i$ annotators. Each annotation $y_i^{(a)} \in \Delta^{C-1}$ represents a probabilistic label over $C$ classes. From these annotations, we derive two complementary supervision signals:
\begin{itemize}
    \item \textbf{Hard labels} $y_i \in \{0,\dots,C-1\}$ obtained via majority vote, used for utility-oriented objectives such as micro-F1.
    \item \textbf{Soft labels} $\bar{y}_i = \frac{1}{A_i}\sum_{a=1}^{A_i} y_i^{(a)} \in \Delta^{C-1}$, which preserve annotator disagreement for calibration.
\end{itemize}

All labels maintain strict instance-level semantics: each instance is associated with exactly one hard label and one soft-label distribution, independent of the number of annotators. 
Ensemble predictions are represented as tensors $\mathbf{P} \in \mathbb{R}^{N \times K \times C}$. 
EDO applies shape normalization to enforce a canonical $[N, K, C]$ representation, ensuring consistent optimization and gradient flow across datasets and implementations.

We construct an ensemble of $K$ pre-trained models $\{f_k\}_{k=1}^K$. The ensemble prediction is defined as a convex combination:
\begin{align}
\hat{y}(x) = \sum_{k=1}^K w_k f_k(x),
\quad w_k \geq 0, \quad \sum_{k=1}^K w_k = 1.
\end{align}
All backbone parameters remain frozen; only $\mathbf{w}$ and the effective ensemble size $K$ are learned.

\subsection{Annotator Supervision Strategies}

Following prior work, we consider two annotator supervision strategies:

\begin{description}
    \item[Random Select.] For each instance $x_i$, an annotator $a \sim \text{Uniform}(1,\dots,A_i)$ is sampled and $y_i^{(a)}$ is used as the training target. This yields an unbiased stochastic estimator of the expected loss under the empirical annotator distribution and requires no annotator metadata at inference time. Random Select constitutes the default supervision strategy in EDO.

    \item[Per-Annotator.] Each ensemble member is paired with a specific annotator $a$ and is trained only on the subset of instances $\mathcal{D}_a = \{ (x_i, y_i^{(a)}) \in \mathcal{D} \}$ for which that annotator provided a label. This preserves annotator-specific biases and provides a structured baseline for analyzing disagreement. It requires $K \leq \min_i A_i$ and access to annotator identity information, and ensemble members may be trained on subsets of different sizes when annotator coverage is incomplete.
\end{description}

Both strategies are optimized within the same joint multi-objective framework described in the following section.

\subsection{Multi-Objective Optimization}
\label{sec:multi-objective}

EDO formulates learning as a joint multi-objective optimization over ensemble weights $\mathbf{w}$ and structure:
\begin{align}
\min_{\mathbf{w}, K} 
\ 
\lambda_{\text{F1}} \mathcal{L}_{\text{F1}}
+ \lambda_{\text{CE}} \mathcal{L}_{\text{CE}}
+ \lambda_{\text{Div}} \mathcal{L}_{\text{Div}}^{(s)}
+ \lambda_{\text{Reg}} \mathcal{L}_{\text{Reg}}
\end{align}
where $\lambda_{(\cdot)} \geq 0$ control trade-offs between predictive utility, calibration, diversity and regularization.

\paragraph{Differentiable micro-F1 loss.}
To approximate the evaluation metric while preserving differentiability, we adopt a soft micro-F1 loss~\citep{Ye:ICML:2012:softF1}:
\begin{align}
\mathcal{L}_{\text{F1}} =
1 - \frac{2 \cdot \mathrm{TP}}{2 \cdot \mathrm{TP} + \mathrm{FP} + \mathrm{FN} + \epsilon}
\end{align}
where $\mathrm{TP}$, $\mathrm{FP}$ and $\mathrm{FN}$ denote soft counts derived from probabilistic predictions and $\epsilon$ ensures numerical stability.

\paragraph{Class-weighted cross-entropy.}
Calibration to soft labels under class imbalance is enforced via:
\begin{align}
\mathcal{L}_{\text{CE}} =
- \frac{1}{N} \sum_{i=1}^N \sum_{c=1}^C
\alpha_c \, y_{i,c} \log(\phi_{i,c})
\end{align}
where $\alpha_c = N / (C \cdot N_c)$ is an inverse-frequency class weight and
$\phi_{i,c} = \sum_{k=1}^K w_k \hat{y}_{i,c}^{(k)}$.
This weighting rescales loss contributions without distorting the empirical annotator distribution.

\paragraph{Reliability-weighted signed diversity loss.}
To explicitly regulate intra-ensemble predictive variability, we define a reliability-weighted pairwise disagreement term:
\begin{align}
\mathcal{L}_{\text{Div}} =
\frac{1}{N} \sum_{i=1}^N \sum_{k,l=1}^K
w_k w_l \left\lVert \hat{y}_{i}^{(k)} - \hat{y}_{i}^{(l)} \right\rVert_1 
\end{align}
EDO treats diversity as a bidirectional modelling signal via a signed objective:
\begin{align}
\mathcal{L}_{\text{Div}}^{(s)} = s \cdot \mathcal{L}_{\text{Div}}, \quad s \in \{-1,+1\},
\end{align}
with fixed $\lambda_{\text{Div}} \geq 0$.
Setting $s=-1$ encourages disagreement, preserving epistemic uncertainty arising from systematic subjectivity, while $s=+1$ suppresses disagreement when the observed divergence is shaped by structural artifacts (e.g., extreme imbalance) rather than coherent subjectivity.
No class-frequency weighting is applied, ensuring that diversity captures internal predictive variability rather than label distribution effects.
This objective is distinct from Manhattan Distance~\citep{Rizzi:COLING:2024}, which evaluates calibration of the aggregated prediction rather than internal ensemble uncertainty.

\paragraph{Regularization.}
We regularize ensemble weights via $\mathcal{L}_{\text{Reg}} = \lVert \mathbf{w} \rVert_2^2$.

\paragraph{Theoretical note.}
The signed diversity objective admits a direct interpretation in terms of predictive variance: as shown in Appendix~\ref{app:theory}, the pairwise dispersion term $L_{\text{Div}}$ is monotonically related to the ensemble’s predictive spread for fixed weights. The sign $s$ therefore selects the direction along the calibration–uncertainty Pareto frontier, allowing EDO to either preserve or reduce epistemic variability without modifying the loss weights themselves.

\subsection{Cross-Entropy Objective Variants}
\label{sec:ce-variants}

We consider three reliability-aware variants of the cross-entropy loss:
\begin{itemize}
    \item \textbf{Random Ensemble:}
    \begin{align}
    \mathcal{L}_{\text{CE}}^{\text{rand}} =
    - \frac{1}{N} \sum_{i,c}
    \alpha_c \, y_{i,c} \log\big( \hat{y}_{i,c}^{(k)} + \epsilon \big)
    \end{align}
    where $k \sim \text{Categorical}(\mathbf{w})$. This yields a stochastic estimator robust to individual annotator signal.

    \item \textbf{Mean Ensemble:}
    \begin{align}
    \mathcal{L}_{\text{CE}}^{\text{mean}} =
    - \frac{1}{N} \sum_{i,c}
    \alpha_c \, y_{i,c} \log\!\left( \sum_{k=1}^K w_k \hat{y}_{i,c}^{(k)} + \epsilon \right)
    \end{align}
    producing stable, consensus-driven calibration.

    \item \textbf{All Ensembles:}
    \begin{align}
    \mathcal{L}_{\text{CE}}^{\text{all}} =
    - \frac{1}{N} \sum_{i,c}
    \alpha_c \, y_{i,c} \sum_{k=1}^K w_k \log\big( \hat{y}_{i,c}^{(k)} + \epsilon \big)
    \end{align}
    preserving member-specific signals while emphasizing reliable models.
\end{itemize}

\subsection{Learning the Ensemble Size}

Rather than fixing ensemble size, EDO learns it end-to-end using a differentiable relaxation with upper bound $K_{\max}$.
Let $\mathcal{K} = \{K_{\min}, \dots, K_{\max}\}$ denote candidate sizes and $\boldsymbol{\pi}$ learnable logits over $\mathcal{K}$.
We apply the Gumbel-Softmax reparameterization~\citep{Jang:2017:gumbelsoftmax}:
\begin{align}
z_j = \log \pi_j + g_j, \quad g_j \sim \mathrm{Gumbel}(0,1),
\end{align}
\begin{align}
\tilde{k}_j = \frac{\exp(z_j / T_e)}{\sum_{j'} \exp(z_{j'} / T_e)}
\end{align}
where temperature $T_e$ is annealed during training.
At inference, we select:
\begin{align}
K^* = \arg\max_j \pi_j, \quad K^* \leq K_{\max}
\end{align}
This enables dataset-adaptive ensemble complexity without combinatorial search while remaining fully differentiable. The complete technical implementation details are provided in Appendix~\ref{app:gumbel}.

A comprehensive summary distinguishing learnable parameters from fixed hyperparameters is provided in Appendix~\ref{app:parameters}.

%% file: scripts/04_experiments.tex
\section{Experiment Setup}
\label{sec:experiments}
This section describes the experimental conditions under which EDO is evaluated. We outline the datasets and preprocessing pipeline, backbone training procedure, optimization and search settings, the evaluation metrics and model-selection criteria used across all experiments.

\subsection{Datasets and Preprocessing}
\label{sec:datasets}

We evaluate EDO on the four datasets from the LeWiDi 2023 shared task benchmark~\citep{LeWiDi:SemEval:2023}: ArMIS (Arabic misogyny), ConvAbuse (English dialogue abuse), HS-Brexit (English Brexit-related hate speech) and MD-Agreement (English offensiveness across topical domains). To ensure comparability with prior work, we use the preprocessed versions of all four datasets provided in the WEL repository~\citep{Huang:WEL:2025}, which include standard cleaning (removal of markup, URLs, user mentions, punctuation, digits and non-ASCII characters) and the official train/dev/test splits.

For ConvAbuse, we follow prior work in mapping the 5-point abuse scale to a binary label ($<0$ = offensive) and flatten each multi-turn conversation into a single sequence before encoding. Table~\ref{tab:data-split} summarizes class distributions, split sizes and annotator coverage.

\begin{table}[ht]
\centering
\caption{Dataset statistics. Neg:Pos reports majority-to-minority ratios. \#Ann denotes the minimum and maximum annotators per instance, with $(n)$ indicating total annotators per dataset.}
\label{tab:data-split}
\resizebox{\columnwidth}{!}{%
\begin{tabular}{@{}lccccc@{}}
\toprule
\textbf{Dataset} & \textbf{Lang} & \textbf{Genre} & \textbf{\#Train/Dev/Test} & \textbf{Neg:Pos} & \textbf{\#Ann} \\
\midrule
ArMIS         & Arabic  & Posts    & 657/141/145      & $1.43{:}1$ & 3 (3) \\
ConvAbuse     & English & Dialogue & 2398/812/840     & $5.16{:}1$ & 2--7 (8) \\
HS-Brexit     & English & Posts    & 784/168/168      & $9.89{:}1$ & 6 (6) \\
MD-Agreement  & English & Posts    & 6592/1104/3057   & $2.36{:}1$ & 5 (670) \\
\bottomrule
\end{tabular}}
\end{table}

\subsection{Backbone Models}
\label{sec:backbones}

We fine-tune \texttt{BERT-base}~\citep{Devlin:BERT:2018} for English datasets and \texttt{AraBERTv2}~\citep{Antoun:2020:ELRA:arabert} for ArMIS, using AdamW (learning rate $2\times10^{-5}$, 500 warmup steps, weight decay in $[0.005, 0.015]$), batch sizes of 16/32 (train/eval), and input sequences truncated to 512 tokens. Training runs for up to 20 epochs with early stopping (patience = 3, $\Delta$F1 = 0.01). After fine-tuning, all backbone parameters are frozen, only ensemble weights and size are learned during EDO optimization. Although we use these models for direct comparability with prior work, EDO is fully model-agnostic and compatible with any set of pre-trained probabilistic classifiers. Additional hyperparameter details are provided in Appendix~\ref{app:backbone-hyperparameters}.



\subsection{EDO Training and Hyperparameters}
All supervision strategies and $\cL_\text{CE}$ aggregation mechanisms used in our experiments follow the formulations given in Section~\ref{sec:method}, and we refer the reader there for full definitions.
Hyperparameters in EDO are tuned on development sets using \textsc{Optuna}~\citep{Akiba:SIGKDD:2019:optuna} with 50 trials per configuration, each run for 10 epochs. The search space comprises loss weights $\lambda_{\text{F1}}, \lambda_{\text{CE}}, \lambda_{\text{Div}} \in [0,1]$, diversity sign $s \in \{-1,+1\}$, $\ell_2$ regularization $\lambda_{\text{Reg}} \in [10^{-5},10^{-2}]$ (log scale), learning rate $\eta \in [10^{-5},10^{-3}]$ (log scale), initial Gumbel temperature $T_0 \in [0.1,1.0]$, temperature decay $\gamma \in [0.01,0.2]$, and ensemble size sampled using the \textit{Random Select} variant with an upper bound of $K_{\max}=10$.

\subsection{Evaluation Metrics}
\label{sec:evaluation}

Evaluation follows standard practice for subjective supervision~\citep{Uma:AAAI:2020,LeWiDi:SemEval:2023,Rizzi:COLING:2024}. We report:
(i) Micro-F1 (F1): task utility from hard predictions; (ii) Cross-Entropy (CE): negative log-likelihood of soft target distributions; and (iii) Manhattan Distance (MD): $\ell_1$ distance between predicted probabilities and empirical annotator distributions. 
For benchmark comparison, we additionally report the Soft Brier Score (BS)~\citep{Flores:AISTATS:2026:bs}, a stricter scoring rule sensitive to overconfident miscalibration. 
For clarity, we include only a brief description of each evaluation metric here, the complete mathematical definitions are provided in Appendix~\ref{app:evaluation-metrics}.


\subsection{Model Selection}
\label{sec:model-selection}
As EDO jointly optimizes competing objectives, we follow Pareto-based model selection via NSGA-II~\citep{Deb:NSGA2:2002}. A configuration is Pareto-optimal if no other configuration improves all three validation metrics (higher F1, lower CE and MD).  
For reporting single values, we select the Pareto-optimal configuration with the lowest CE, while visualizing full trade-offs in Section~\ref{sec:results}.

%% file: scripts/05_results.tex
\input{tables/ablation_tables}

\input{tables/heatmap_results}

\section{Results}
\label{sec:results}

We evaluate EDO through ablations, structural analyses and comparisons with established baselines. The goal is to understand how signed disagreement control ($s$), cross-entropy aggregation, reliability-aware weighting and learned ensemble structure jointly influence utility (F1), calibration (CE) and alignment with annotator distributions (MD).

\subsection{Effect of Signed Diversity and Cross-Entropy Objective Aggregation}
\label{sec:ablation}

All ablations use the \textit{Random Select} supervision strategy. We examine:
(i) the three aggregation variants 
$\mathcal{L}_{\text{CE}}^{\text{rand}}$, 
$\mathcal{L}_{\text{CE}}^{\text{mean}}$, and 
$\mathcal{L}_{\text{CE}}^{\text{all}}$, and
(ii) the diversity sign $s \in \{-1,+1\}$, corresponding to 
disagreement-preserving ($s=-1$) and disagreement-suppressing ($s=+1$) regimes.  
Loss weights are fixed to 
$\lambda_{\text{F1}}, \lambda_{\text{CE}}, \lambda_{\text{Div}}=1$ and 
$\lambda_{\text{Reg}}=10^{-3}$, with identical optimization hyperparameters 
($\eta=10^{-3}$, $T_0=0.5$, $\gamma=0.05$).

\paragraph{Aggregation variants under signed diversity.}
Table~\ref{tab:signed-diversity-ce-variant} reports test F1, CE, and $\Delta$CE for all variants under both disagreement regimes. Full metrics, including MD and $\Delta$F1/$\Delta$MD, are provided in Appendix~\ref{app:generalisation}.  
Under moderate imbalance (ArMIS, MD-Agreement), 
$\mathcal{L}_{\text{CE}}^{\text{mean}}$ consistently provides the most stable calibration, achieving the lowest or near-lowest CE and small generalization gaps. Sensitivity to $s$ is limited: MD-Agreement shows near-identical performance across signs, while on ArMIS disagreement preservation ($s=-1$) slightly benefits $\mathcal{L}_{\text{CE}}^{\text{all}}$, whereas disagreement suppression ($s=+1$) favors $\mathcal{L}_{\text{CE}}^{\text{mean}}$.  
ConvAbuse shows weak dependence on $s$, consistent with discourse-level subjectivity where disagreement reflects interpretation rather than structural artifacts.  
In contrast, under severe imbalance (HS-Brexit), the effect of $s$ becomes critical: disagreement preservation ($s=-1$) benefits $\mathcal{L}_{\text{CE}}^{\text{rand}}$, improving both F1 and CE, while disagreement suppression ($s=+1$) substantially reduces CE for $\mathcal{L}_{\text{CE}}^{\text{mean}}$ and $\mathcal{L}_{\text{CE}}^{\text{all}}$, though sometimes with larger $\Delta$CE.

\paragraph{Signed diversity as complementary regularization.}
Table~\ref{tab:diversity-ablation} reports the marginal contribution of 
$\mathcal{L}_{\text{Div}}^{(s)}$ across datasets using 
$\mathcal{L}_{\text{CE}}^{\text{all}}$.  
Although diversity yields only modest gains under single-objective losses ($\Delta$F1 = +0.0093 for $s=+1$ and +0.0062 for $s=-1$ under $\mathcal{L}_{\text{F1}}$; +0.0044 for $s=+1$ under $\mathcal{L}_{\text{CE}}$), its impact becomes more evident when combined with the joint objective 
$\mathcal{L}_{\text{F1}} + \mathcal{L}_{\text{CE}}$.  
Notably, disagreement preservation ($s=-1$) with the joint objective yields the most substantial improvement in calibration ($\Delta$CE = -0.0993), alongside modest F1 gains ($\Delta$F1 = +0.0043), demonstrating that diversity operates most effectively as a complementary mechanism that reshapes the utility–calibration trade-off rather than as a standalone objective. Agreement promotion ($s=+1$) with the joint objective, by contrast, yields more modest changes ($\Delta$F1 = +0.0017, $\Delta$CE = +0.0107) while reducing MD ($\Delta$MD = -0.0093). The per-dataset contribution can be found in Appendix~\ref{app:diversity-ablation}.

\paragraph{Interpretation.}
These results indicate a non-trivial interaction between aggregation strategy and disagreement control. It is consistent with our formulation, and disagreement preservation (s = -1) helps when variation reflects subjectivity.


\input{tables/ensemble_weights}


\subsection{Effect of Reliability‑Aware Ensemble Weighting}

Table~\ref{tab:effect-ensemble-weights} isolates the impact of reliability-aware ensemble weighting by comparing learned, reliability-weighted aggregation against uniform averaging (unweighted). Results correspond to Pareto-optimal configurations selected by minimum CE among non-dominated solutions (Section~\ref{sec:model-selection}). Across all benchmarks, learned weights consistently reduce CE (negative $\Delta$CE), with substantial gains in several settings (e.g., $\Delta$CE = $-0.5607$ for EDO-PerAnn on MD-Agreement). This indicates that down-weighting uncertain or noisy annotators directly improves probabilistic calibration without sacrificing predictive alignment. Calibration improvements are accompanied by reduced MD in multiple cases, particularly for EDO-PerAnn, suggesting that reliability weighting complements annotator-specific modeling by filtering structural noise.

Effects on F1 are modest: EDO-Random shows small decreases in some settings, whereas EDO-PerAnn benefits on MD-Agreement (+0.0350). Overall, adaptive weighting primarily enhances uncertainty quality with minimal impact on point-estimate utility.

It is important to note that Table~\ref{tab:effect-ensemble-weights} reports the \textit{marginal} effect of reliability weighting in isolation. As detailed in the per-dataset ablation (Appendix~\ref{app:diversity-ablation}), the full performance gains of EDO arise from \textit{synergistic joint optimization} of all components. While weighting improves calibration independently, the substantial CE reductions in Table~\ref{tab:benchmark-results} emerge when weighting is combined with differentiable ensemble-size selection, class-weighted cross-entropy for imbalance, and the signed diversity regularizer. In this joint regime, diversity does not act as a standalone utility driver but rather as a directional regularizer that prevents ensemble collapse and preserves meaningful epistemic variation. By controlling intra-ensemble disagreement, diversity creates a stable optimization landscape that allows reliability weighting and adaptive $K$ to operate effectively, ultimately navigating the utility-calibration Pareto frontier.


\subsection{Sensitivity of Objective Weights and Structural Parameters}
\label{sec:sensitivity}
We next examine how optimization hyperparameters and learned structure affect model behavior under joint multi-objective optimization. While Spearman correlations~\citep{Kendall:Spearman:1969,Zwillinger:SpearmanCorr:1999} capture marginal sensitivities, they reveal how EDO navigates the utility-calibration Pareto frontier when competing gradients are balanced simultaneously. Figure~\ref{fig:corr-heatmap} presents correlations under disagreement preservation ($s=-1$, \ref{fig:corr-neg}) and suppression ($s=+1$, \ref{fig:corr-pos}).

\paragraph{Joint steering via objective weights.}
The utility weight $\lambda_{\text{F1}}$ correlates positively with F1 and negatively with CE and BS across most datasets and both diversity signs, indicating that emphasizing utility shifts the optimizer toward higher-performing regions without systematically degrading calibration. Conversely, $\lambda_{\text{CE}}$ exhibits weak and inconsistent correlations with CE. This reflects EDO's synergistic gradient balancing: calibration arises from the interaction of $\lambda_{\text{CE}}$ with diversity regularization and reliability weighting, rather than being driven by $\lambda_{\text{CE}}$ in isolation. The regularization weight $\lambda_{\text{Reg}}$ shows consistent negative correlations with CE, suggesting that $\ell_2$ stabilization prevents weight collapse and supports calibration robustness.

\paragraph{Diversity as a directional regularizer.}
The diversity weight $\lambda_{\text{Div}}$ displays sign-dependent behavior, suggesting its role as a steering mechanism. Under disagreement preservation ($s=-1$), $\lambda_{\text{Div}}$ correlates positively with CE, empirically quantifying the utility-calibration trade-off. Flipping the sign to $s=+1$ dampens these correlations, demonstrating that $\lambda_{\text{Div}}$ modulates the joint gradient flow along the calibration axis.

\paragraph{Structural adaptation and optimization dynamics.}
Unlike fixed hyperparameters, the ensemble size $K$ is a learned structural variable responding directly to joint optimization pressure. $K$ correlates strongly and negatively with CE and BS on all datasets, indicating that the optimizer allocates additional members precisely when doing so improves calibration. Positive correlations with F1 show that complexity is retained only when it benefits the joint objective. Optimization parameters $T_0$ and $\gamma$ exhibit dataset-dependent sensitivities, reflecting how discrete exploration interacts with annotator structure, though their effects remain secondary to the dominant signal from $K$.

Collectively, these correlations illustrate that EDO's components function as coordinated controls within a shared optimization landscape. Weak marginal effects of $\lambda_{\text{CE}}$, sign-modulated $\lambda_{\text{Div}}$, and robust $K$ adaptation suggest that utility, calibration, and diversity gradients interact synergistically to handle dataset-specific annotation regimes.

\input{tables/benchmark_results}

\subsection{Comparison with Baselines}
\label{sec:benchmark}
To ensure comparability, we use the Soft‑CE, Top‑5 Voting and WEL results reported in the WEL paper and additionally include the Soft-MD baseline and BS. The four baselines are: (i) single‑model soft‑label CE (Soft‑CE)~\citep{Uma:AAAI:2020}, (ii) single‑model soft‑label MD (Soft‑MD)~\citep{Rizzi:COLING:2024}, (iii) majority‑voting ensemble (Top‑5 Voting)~\citep{Xu:ICNLSP:2024}, and (iv) WEL~\citep{Huang:WEL:2025}. Table~\ref{tab:benchmark-results} reports test performance.

EDO‑Random obtains the lowest CE on all datasets, reducing CE by up to 78\% relative to Soft‑CE (ConvAbuse: 0.2149 vs.\ 0.9671) and 62\% relative to WEL (ConvAbuse: 0.2149 vs.\ 0.5577), indicating superior calibration to annotator distributions. These gains are corroborated by the Soft Brier Score: EDO‑Random achieves the lowest BS on all four benchmarks (e.g., 0.0640 vs.\ 0.0699 for WEL on ConvAbuse; 0.1086 vs.\ 0.1606 on MD‑Agreement), suggesting that improved calibration is robust across metrics. Across datasets, BS and CE reductions are strongly correlated ($\rho = 0.94$), suggesting both metrics capture complementary aspects of probabilistic alignment with annotator uncertainty. Although WEL reaches the highest F1 scores (consistent with its utility‑focused objective), it consistently shows higher CE and BS than EDO‑Random, suggesting the trade‑off between utility and calibration.

EDO-PerAnn achieves the lowest MD on ConvAbuse and HS-Brexit, indicating benefits when annotator identities correspond to meaningful and persistent perspectives. In these datasets, each annotator contributes enough labels to encode stable interpretive tendencies, and pairing ensemble members with specific annotators allows EDO to capture these structured differences. In contrast, datasets with sparse or highly constrained annotator pools, such as MD-Agreement (670 annotators with heterogeneous coverage) and ArMIS (only 3 total annotators), show diminished returns from Per-Annotator supervision. In these settings, fixed annotator assignments either fragment the data into weakly informative subsets (MD-Agreement) or restrict ensemble diversity due to the minimal member pool (ArMIS), resulting in poorer calibration (e.g., CE = 0.6508 vs.\ 0.5094 and BS = 0.1592 vs.\ 0.1086 on MD-Agreement; CE = 0.6562 vs.\ 0.5719 and BS = 0.2626 vs.\ 0.2183 on ArMIS). EDO-Random, which draws repeatedly from the pooled annotator distribution, is therefore more robust when individual annotators do not supply coherent perspective-specific signals or when the fixed pool is too small to sustain meaningful diversity.

\paragraph{Stability.}
Table~\ref{tab:variant-results-std} reports standard deviations over five runs. EDO-Random exhibits low variance in F1, CE, MD, and BS on most datasets (all $\sigma < 0.025$), indicating stable predictive and distributional behavior. Higher MD variance on HS-Brexit reflects the extreme class imbalance that affects all evaluated methods. EDO-PerAnn shows slightly elevated CE variability on ArMIS ($\sigma = 0.0512$), which is expected given its small and fixed annotator pool, but remains stable elsewhere. Low BS variance for both variants (all $\sigma \leq 0.0135$) suggests consistent calibration across random seeds, reinforcing that EDO's joint optimization produces reliable uncertainty estimates.

%% file: tables/ablation_tables.tex
\begin{table}[htbp]
\centering
\caption{Test performance of cross-entropy aggregation variants 
($\mathcal{L}_{\text{CE}}^{\text{rand}}$, 
$\mathcal{L}_{\text{CE}}^{\text{mean}}$, 
$\mathcal{L}_{\text{CE}}^{\text{all}}$) 
under disagreement-preserving ($s=-1$) and disagreement-suppressing ($s=+1$) regimes.
Metrics include F1 ($\uparrow$), CE ($\downarrow$), and generalization gap $\Delta$CE = test - dev ($\downarrow$). Full results including Manhattan Distance (MD) are reported in Appendix~\ref{app:generalisation}.}

\label{tab:signed-diversity-ce-variant}
\resizebox{0.5\textwidth}{!}{%
\begin{tabular}{@{}llcccc:cc@{}}
\toprule
\multirow{2}{*}{Dataset} & \multirow{2}{*}{Variant} &
\multicolumn{2}{c}{Test F1} &
\multicolumn{2}{c}{Test CE} &
\multicolumn{2}{:c}{$\Delta$CE} \\
\cmidrule(lr){3-4} \cmidrule(lr){5-6} \cmidrule(lr){7-8}
& & $s=-1$ & $s=+1$ & $s=-1$ & $s=+1$ & $s=-1$ & $s=+1$ \\
\midrule

\multirow{3}{*}{ArMIS}
& Random & 0.7655 & 0.7034 & 0.6097 & 0.6796 & -0.0347 & +0.0347 \\
& Mean   & 0.7655 & \textbf{0.7310} & 0.6054 & \textbf{0.5992} & -0.0655 & -0.0322 \\
& All    & \textbf{0.7724} & 0.7103 & \textbf{0.5724} & 0.6007 & -0.0614 & -0.0304 \\
\hdashline

\multirow{3}{*}{ConvAbuse}
& Random & 0.9202 & \textbf{0.9321} & \textbf{0.2409} & 0.2483 & +0.0209 & +0.0374 \\
& Mean   & 0.9214 & 0.9298 & 0.2482 & 0.2415 & +0.0345 & +0.0309 \\
& All    & \textbf{0.9286} & 0.9298 & 0.2943 & \textbf{0.2308} & +0.0687 & +0.0186 \\
\hdashline

\multirow{3}{*}{HS-Brexit}
& Random & \textbf{0.9107} & 0.8929 & \textbf{0.3465} & 0.4490 & +0.0177 & +0.1131 \\
& Mean   & 0.8988 & \textbf{0.9107} & 0.6617 & 0.3736 & +0.3198 & +0.0317 \\
& All    & 0.8988 & 0.8988 & 0.4601 & \textbf{0.3494} & +0.1074 & +0.0240 \\
\hdashline

\multirow{3}{*}{MD-Agreement}
& Random & 0.8106 & 0.7998 & 0.5182 & 0.5325 & +0.0064 & +0.0196 \\
& Mean   & \textbf{0.8178} & \textbf{0.8132} & \textbf{0.5163} & \textbf{0.5167} & +0.0040 & +0.0081 \\
& All    & 0.8054 & 0.7880 & 0.5244 & 0.5278 & +0.0100 & +0.0110 \\
\bottomrule
\end{tabular}%
}
\end{table}

\begin{table}[htbp]
\centering
\caption{
Mean marginal effect of signed diversity regularization 
$\mathcal{L}_{\text{Div}}^{(s)}$ across datasets using 
$\mathcal{L}_{\text{CE}}^{\text{all}}$.
Values denote dataset-averaged 
$\Delta = \text{with } \mathcal{L}_{\text{Div}}^{(s)} - \text{without } \mathcal{L}_{\text{Div}}^{(s)}$.
Positive $\Delta$F1 and negative $\Delta$CE/$\Delta$MD indicate improvement.}
\label{tab:diversity-ablation}
\small
\resizebox{0.5\textwidth}{!}{%
\begin{tabular}{@{} l *{3}{c c} @{}}
\toprule
\multirow{2}{*}{Base Objective} & \multicolumn{2}{c}{$\Delta$F1} & \multicolumn{2}{c}{$\Delta$CE} & \multicolumn{2}{c}{$\Delta$MD} \\
\cmidrule(lr){2-3} \cmidrule(lr){4-5} \cmidrule(lr){6-7}
& $s=+1$ & $s=-1$ & $s=+1$ & $s=-1$ & $s=+1$ & $s=-1$ \\
\midrule
$\mathcal{L}_{\text{F1}}$ & \textbf{+0.0093} & \textbf{+0.0062} & +0.0288 & +0.0202 & \textbf{-0.0618} & \textbf{+0.0044} \\
$\mathcal{L}_{\text{CE}}$ & +0.0044 & -0.0016 & \textbf{-0.0356} & -0.0192 & +0.0100 & +0.0235 \\
$\mathcal{L}_{\text{F1}}+\mathcal{L}_{\text{CE}}$ & +0.0017 & +0.0043 & +0.0107 & \textbf{-0.0993} & -0.0093 & +0.0100 \\
\bottomrule
\end{tabular}
}
\end{table}

%% file: tables/heatmap_results.tex

\begin{figure*}[t]
\centering
\begin{subfigure}[b]{0.95\textwidth}
\centering
\includegraphics[width=\textwidth]{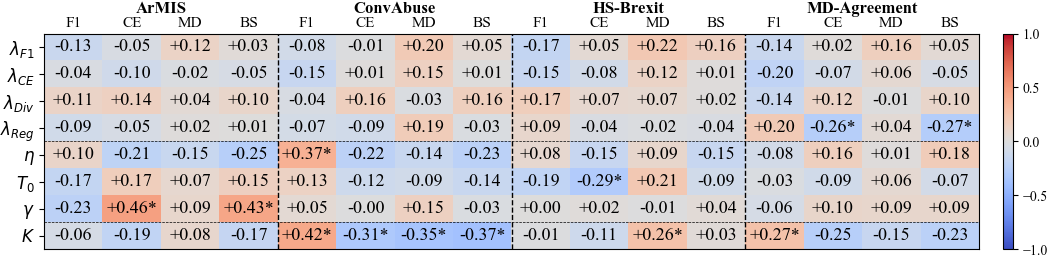}
\caption{Disagreement preservation ($s=-1$)}
\label{fig:corr-neg}
\end{subfigure}

\medskip

\begin{subfigure}[b]{0.95\textwidth}
\centering
\includegraphics[width=\textwidth]{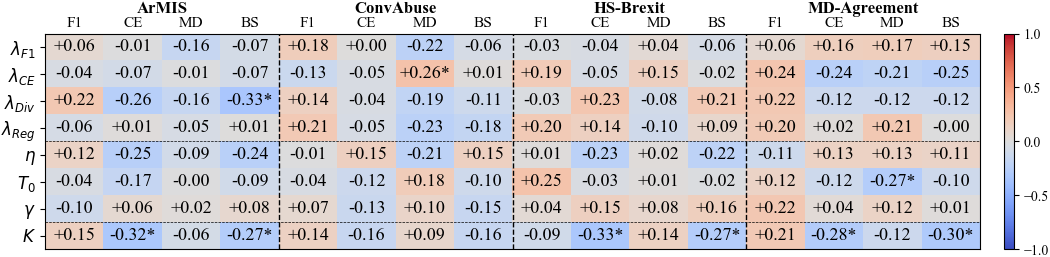}
\caption{Disagreement suppression ($s=+1$)}
\label{fig:corr-pos}
\end{subfigure}
\caption{Spearman rank correlations between hyperparameters and development-set metrics (F1, CE, MD, BS). Positive correlations with F1 indicate improved utility; negative correlations with CE/MD/BS indicate improved calibration and reduced divergence from annotator distributions. Statistically significant correlations ($p<0.05$) marked with $^{*}$.}
\label{fig:corr-heatmap}
\end{figure*}

%% file: tables/ensemble_weights.tex
\begin{table}[htbp]
\centering
\caption{Effect of reliability-aware ensemble weighting (weighted) compared to uniform averaging (unweighted).
Results correspond to Pareto-optimal configurations selected by minimum CE among non-dominated solutions.
$\Delta$ = weighted - unweighted (\colorbox{improvegreen}{Green} indicates improvement).}
\label{tab:effect-ensemble-weights}
\resizebox{0.48\textwidth}{!}{%
\begin{tabular}{lcccccc}
\toprule
Variant & \multicolumn{3}{c}{EDO-Random} & \multicolumn{3}{c}{EDO-PerAnn} \\
\cmidrule(lr){2-4} \cmidrule(lr){5-7}
Dataset & $\Delta$F1 & $\Delta$CE & $\Delta$MD & $\Delta$F1 & $\Delta$CE & $\Delta$MD \\
\midrule
ArMIS & 
\cellcolor{degradered}-0.0069 & 
\cellcolor{improvegreen}-0.0095 & 
\cellcolor{degradered}+0.0132 & 
\cellcolor{degradered}-0.0069 & 
\cellcolor{improvegreen}-0.0005 & 
\cellcolor{degradered}+0.0023 \\
ConvAbuse & 
\cellcolor{degradered}-0.0036 & 
\cellcolor{improvegreen}-0.0099 & 
\cellcolor{improvegreen}-0.0028 & 
\cellcolor{improvegreen}+0.0108 & 
\cellcolor{improvegreen}-0.0663 & 
\cellcolor{degradered}+0.0011 \\
HS-Brexit & 
\cellcolor{degradered}-0.0119 & 
\cellcolor{improvegreen}-0.0530 & 
\cellcolor{degradered}+0.0409 & 
0.0000 & 
\cellcolor{improvegreen}-0.0029 & 
\cellcolor{improvegreen}-0.0079 \\
MD-Agreement & 
\cellcolor{improvegreen}+0.0036 & 
\cellcolor{improvegreen}-0.2756 & 
\cellcolor{improvegreen}-0.0026 & 
\cellcolor{improvegreen}+0.0350 & 
\cellcolor{improvegreen}-0.5607 & 
\cellcolor{improvegreen}-0.0906 \\
\bottomrule
\end{tabular}%
}
\end{table}

%% file: tables/benchmark_results.tex
\begin{table*}[t]
\centering
\caption{Comparison with Soft-CE~\citep{Uma:AAAI:2020}, 
Soft-MD~\citep{Rizzi:COLING:2024}, 
Top-5 Voting~\citep{Xu:ICNLSP:2024}, 
and WEL (Random)~\citep{Huang:WEL:2025}.
Metrics include F1 ($\uparrow$), CE ($\downarrow$), MD ($\downarrow$) and BS ($\downarrow$).}
\label{tab:benchmark-results}
\resizebox{0.99\textwidth}{!}{%
\begin{tabular}{@{}l:cccc:cccc:cccc:cccc@{}}
\toprule
Dataset & \multicolumn{4}{c}{ArMIS} & \multicolumn{4}{c}{ConvAbuse} & \multicolumn{4}{c}{HS-Brexit} & \multicolumn{4}{c}{MD-Agreement} \\
\cmidrule(lr){2-5} \cmidrule(lr){6-9} \cmidrule(lr){10-13} \cmidrule(lr){14-17}
Metric & F1 & CE & MD & BS & F1 & CE & MD & BS & F1 & CE & MD & BS & F1 & CE & MD & BS \\
\midrule
Soft-CE & 0.6596 & 0.8039 & 0.7144 & 0.4177 & 0.8362 & 0.9671 & 4.8068 & 0.3408 & 0.7917 & 0.7652 & 0.7985 & 0.2545 & 0.7880 & 0.9948 & 1.7574 & 0.3693 \\
Soft-MD & 0.6159 & 0.6594 & 0.6755 & 0.2996 & 0.8893 & 0.6233 & 0.2198 & 0.1483 & 0.8929 & 0.6772 & 0.2640 & 0.1368 & 0.6804 & 0.7113 & 0.6146 & 0.3501 \\
Top-5 Voting & 0.7310 & 0.6529 & 0.5498 & 0.2554 & 0.9310 & 0.5651 & 0.1648 & 0.0677 & 0.8929 & 0.6154 & 0.2394 & 0.1030 & 0.7808 & 0.6629 & 0.3995 & 0.1641 \\
WEL & \textbf{0.7793} & 0.6385 & \textbf{0.5028} & 0.2653 & \textbf{0.9405} & 0.5577 & 0.1709 & 0.0699 & \textbf{0.9167} & 0.5889 & 0.2585 & 0.0861 & \textbf{0.8214} & 0.6245 & \textbf{0.3632} & 0.1606 \\ \hdashline
EDO\footnotesize{-Random} & 0.7572 & \textbf{0.5719} & 0.5370 & \textbf{0.2183} & 0.9331 & \textbf{0.2149} & 0.1850 & \textbf{0.0640} & 0.8905 & \textbf{0.3441} & 0.2714 & \textbf{0.0803} & 0.8152 & \textbf{0.5094} & 0.3649 & \textbf{0.1086} \\
EDO\footnotesize{-PerAnn} & 0.7366 & 0.6562 & 0.5594 & 0.2626 & 0.9295 & 0.2993 & \textbf{0.1553} & 0.0725 & 0.8929 & 0.4412 & \textbf{0.2367} & 0.0933 & 0.7894 & 0.6508 & 0.4030 & 0.1592 \\
\bottomrule

\end{tabular}%
}
\end{table*}

\begin{table*}[bht]
\centering
\caption{Mean standard deviation over 5 random seeds. Color coding: \textcolor{mutedgreen}{low} ($\sigma$<0.02), \textcolor{mutedyellow}{medium} (0.02$\leq$$\sigma$$<$0.05), \textcolor{mutedred}{high} ($\sigma$$\geq$0.05).}
\label{tab:variant-results-std}
\resizebox{0.99\textwidth}{!}{%
\begin{tabular}{@{}l:cccc:cccc:cccc:cccc@{}}
\toprule
Dataset & \multicolumn{4}{c:}{ArMIS} & \multicolumn{4}{c:}{ConvAbuse} & \multicolumn{4}{c:}{HS-Brexit} & \multicolumn{4}{c}{MD-Agreement} \\ 
\cmidrule(lr){2-5} \cmidrule(lr){6-9} \cmidrule(lr){10-13} \cmidrule(lr){14-17}
Variant & F1 std & CE std & MD std & BS std & F1 std & CE std & MD std & BS std & F1 std & CE std & MD std & BS std & F1 std & CE std & MD std & BS std \\ \midrule

EDO\footnotesize{-Random}
& \textcolor{mutedgreen}{0.0133}
& \textcolor{mutedgreen}{0.0031}
& \textcolor{mutedyellow}{0.0234}
& \textcolor{mutedgreen}{0.0025}
& \textcolor{mutedgreen}{0.0023}
& \textcolor{mutedgreen}{0.0005}
& \textcolor{mutedgreen}{0.0062}
& \textcolor{mutedgreen}{0.0004}
& \textcolor{mutedgreen}{0.0068}
& \textcolor{mutedgreen}{0.0045}
& \textcolor{mutedyellow}{0.0249}
& \textcolor{mutedgreen}{0.0025}
& \textcolor{mutedgreen}{0.0017}
& \textcolor{mutedgreen}{0.0005}
& \textcolor{mutedgreen}{0.0036}
& \textcolor{mutedgreen}{0.0003}
\\

EDO\footnotesize{-PerAnn}
& \textcolor{mutedgreen}{0.0031}
& \textcolor{mutedred}{0.0512}
& \textcolor{mutedyellow}{0.0221}
& \textcolor{mutedgreen}{0.0135}
& \textcolor{mutedgreen}{0.0018}
& \textcolor{mutedyellow}{0.0262}
& \textcolor{mutedgreen}{0.0040}
& \textcolor{mutedgreen}{0.0021}
& \textcolor{mutedgreen}{0.0000}
& \textcolor{mutedgreen}{0.0098}
& \textcolor{mutedgreen}{0.0006}
& \textcolor{mutedgreen}{0.0009}
& \textcolor{mutedgreen}{0.0026}
& \textcolor{mutedgreen}{0.0109}
& \textcolor{mutedgreen}{0.0040}
& \textcolor{mutedgreen}{0.0036}
\\
\bottomrule
\end{tabular}%
}
\end{table*}

%% file: scripts/06_limitations.tex
\section{Limitations and Future Work}
\label{sec:limitations}
While EDO provides a flexible and empirically effective framework for modeling annotator disagreement, several factors shape its behavior and motivate future research.

\paragraph{Dependence on the structure of disagreement.}
The effect of the signed diversity objective is inherently dataset-dependent: preserving disagreement is most effective when variation reflects genuine subjective interpretations, whereas reducing divergence is preferable when the observed variation is structurally induced (e.g., extreme imbalance). Although EDO adapts to these regimes, it does not distinguish principled subjective variation from artifact-driven divergence caused by sparse coverage or inconsistent annotator behavior. This limits robustness under distribution shift or in settings that mix coherent and unreliable annotation signals.

\paragraph{Interaction with multi-objective optimization.}
EDO jointly optimizes utility, calibration and internal diversity, and the interaction among these objectives can vary across datasets. Diversity primarily reshapes the calibration–utility balance rather than acting as a standalone objective, and hyperparameters such as $\lambda_{\text{Div}}$ and $T_{0}$ can influence this balance in dataset-specific ways. Pareto-based model selection mitigates this sensitivity, but future work may explore dynamic weighting or gradient-normalization strategies that adaptively stabilize objective scales during training.

\paragraph{Architectural and supervision constraints.}
The current implementation operates in prediction space with homogeneous frozen backbones, which provides efficiency and isolates annotator-driven uncertainty but limits representational flexibility. 
Annotator-specific supervision is effective when annotator identities reflect stable perspectives, but less reliable under sparse or heterogeneous annotator pools. 
Incorporating lightweight annotator representations or heterogeneous experts may increase expressiveness without sacrificing efficiency.

\paragraph{Future directions.}
Promising extensions include integrating adapters or partially trainable components to enrich the ensemble's hypothesis space; leveraging annotator metadata when available to inform source-aware diversity control; and developing mechanisms to infer the structure of disagreement, for example, through entropy analysis, clustering or latent-variable models. These directions may further improve robustness and adaptability across diverse subjective supervision regimes.

\paragraph{Practical implications.}
Because EDO optimizes only prediction-space components and leaves backbones fixed, it is suitable for resource-constrained pipelines and frozen-LM deployments. The framework offers a lightweight means of improving calibration and uncertainty representation while remaining compatible with existing model infrastructures. Additional runtime and memory details are provided in Appendix~\ref{app:efficiency}.

Collectively, these limitations highlight opportunities for extending EDO toward richer representations of annotator behavior and more adaptive multi-objective coordination.

%% file: scripts/07_conclusions.tex
\section{Conclusions}
\label{sec:conclusion}
We presented Ensemble Diversity Optimization (EDO), a prediction-space framework that uses signed diversity, adaptive weighting and learned ensemble structure to model annotator disagreement in subjective supervision. 
EDO adjusts internal disagreement to the structure of annotator variation, preserving subjective signal when present and reducing artifact-driven divergence.
Across four benchmarks, EDO delivers substantial gains in calibration and alignment with annotator distributions while maintaining competitive utility. These results suggest that annotator disagreement carries a structured signal rather than noise, and that EDO offers an efficient and principled foundation for uncertainty-aware learning in settings with inherently variable human judgements.

%% file: scripts/E_acknowledge.tex

\subsubsection*{Acknowledgements}
The authors would like to thank the anonymous reviewers for their valuable comments and constructive feedback, which have greatly improved the quality of this paper. 

%% file: scripts/supplement.tex
\newpage

\onecolumn


\appendix

\section{Theoretical Note: Signed Diversity and Predictive Dispersion}
\label{app:theory}

The signed diversity objective admits a direct interpretation in terms of ensemble predictive dispersion and its role within multi-objective optimization.

\begin{theorem}[Signed diversity and predictive dispersion]
Let $\{\hat{y}^{(k)}(x)\}_{k=1}^K \subset \Delta^{C-1}$ denote ensemble member predictions for input $x$, and let
\begin{align}
\hat{y}(x)=\sum_{k=1}^K w_k \hat{y}^{(k)}(x), 
\quad 
w_k \ge 0, 
\quad 
\sum_{k=1}^K w_k = 1
\end{align}
be the reliability-weighted ensemble prediction. Define the diversity functional
\begin{align}
\mathcal{L}_{\mathrm{Div}}(x)
=
\sum_{k,l} w_k w_l 
\left\lVert
\hat{y}^{(k)}(x) - \hat{y}^{(l)}(x)
\right\rVert_1 .
\end{align}

Then $\mathcal{L}_{\mathrm{Div}}(x)$ is proportional to a weighted measure of ensemble predictive dispersion and satisfies
\begin{align}
\sum_{k} w_k 
\left\lVert
\hat{y}^{(k)}(x)-\hat{y}(x)
\right\rVert_1
\;\le\;
\mathcal{L}_{\mathrm{Div}}(x)
\;\le\;
2 \sum_{k} w_k 
\left\lVert
\hat{y}^{(k)}(x)-\hat{y}(x)
\right\rVert_1 .
\end{align}

Consequently, minimizing $-\mathcal{L}_{\mathrm{Div}}$ (signed diversity $s=-1$) increases ensemble dispersion, whereas minimizing $+\mathcal{L}_{\mathrm{Div}}$ ($s=+1$) reduces dispersion.

\end{theorem}

\begin{proof}[Proof sketch]
The pairwise dispersion can be related to deviations from the barycenter
$\bar a=\sum_k w_k a_k$ using triangle inequality:
\[
\lVert a_k-a_l\rVert_1
\le
\lVert a_k-\bar a\rVert_1
+
\lVert a_l-\bar a\rVert_1 .
\]
Multiplying by $w_k w_l$ and summing over $(k,l)$ yields the upper bound,
while convexity of the $\ell_1$ norm yields the lower bound.
Substituting $a_k=\hat y^{(k)}(x)$ completes the argument.
\end{proof}

\paragraph{Interpretation \& Connection to EDO.}
While the bound follows from standard convexity and the triangle inequality~\citep{Boyd:book:2004:convex}, its value for EDO lies in formally establishing that the pairwise $\ell_1$ diversity term used in the loss \textit{directly controls dispersion around the reliability-weighted barycenter}. This provides a mathematically grounded, differentiable knob for regulating intra-ensemble spread without relying on post-hoc heuristics or fixed architectures~\citep{Wood:JMLR:2023}.

\paragraph{How Reliability Weights Express Uncertainty.}
Unlike explicit uncertainty quantification methods (e.g., predictive entropy, Monte Carlo dropout~\citep{Singh:EMNLP:2012}, or variance pooling~\citep{Fort:2019:deepensembles}), EDO's weights $w_k$ are not precomputed scores. Instead, they emerge \textit{implicitly} from joint gradient-based optimization of utility (soft F1), calibration (class-weighted CE), diversity, and L2 regularization. During training:
\begin{itemize}
    \item Members whose predictions consistently align with the annotator distribution and reduce calibration error receive higher weights.
    \item Members that are overconfident, miscalibrated, or introduce unstructured disagreement are suppressed via the combined pressure of CE and the diversity regularizer.
    \item The final weighted prediction $\hat{y}(x)$ therefore expresses uncertainty not through a separate metric, but through the learned allocation of probability mass across members conditioned on the annotator distribution and the chosen diversity direction.
\end{itemize}

\paragraph{Joint Optimization Dynamics \& Scope.}
The theorem assumes fixed $w_k$, but in EDO weights and effective cardinality $K$ are updated simultaneously via backpropagation. Consequently, predictive dispersion is a \textit{dynamic} property shaped by the trade-off surface between utility, calibration, and signed diversity. The sign parameter $s \in \{-1,+1\}$ does not alter objective magnitudes; it selects the optimization direction along the dispersion axis:
\begin{itemize}
    \item $s=-1$ encourages the optimizer to tolerate higher intra-ensemble variance, preserving epistemic variability that reflects genuine annotator subjectivity.
    \item $s=+1$ penalizes dispersion, filtering structural noise and promoting consensus when disagreement is driven by class imbalance or sparse annotator coverage.
\end{itemize}
This analytical grounding justifies EDO's design: rather than treating diversity as an auxiliary heuristic or fixing ensemble structure a priori, we integrate it directly into a differentiable multi-objective loss where reliability weighting, calibration, and utility are co-optimized: a formulation consistent with gradient-based Pareto navigation in machine learning \citep{Sener:NIPS:2018}. Empirical results in Section~\ref{sec:experiments} demonstrate that this joint formulation consistently outperforms baselines that optimize these components independently or sequentially.

\section{Differentiable Learning of Ensemble Size via Gumbel-Softmax}
\label{app:gumbel}

The optimal ensemble size $K$ is typically dataset-dependent. Treating $K$ as a discrete hyperparameter would require combinatorial search, which is incompatible with gradient-based training. To enable end-to-end optimization, EDO learns $K$ through a \emph{differentiable relaxation} of discrete selection using the Gumbel--Softmax trick~\citep{Jang:2017:gumbelsoftmax}.

We consider a candidate set $\mathcal{K} = \{K_{\min}, K_{\min}+1, \dots, K_{\max}\}$ with learnable logits $\boldsymbol{\pi} = [\pi_1, \dots, \pi_M]$, where $M = |\mathcal{K}|$. Direct sampling of $k \in \mathcal{K}$ is non-differentiable, so we apply the reparameterization
\begin{align}
z_j = \log \pi_j + g_j, \qquad g_j \sim \text{Gumbel}(0,1),
\end{align}
followed by a softmax with temperature $T_e > 0$
\begin{align}
\tilde{k}_j = \frac{\exp(z_j / T_e)}{\sum_{j'=1}^M \exp(z_{j'} / T_e)} .
\end{align}

When $T_e$ is high, $\tilde{\mathbf{k}}$ is smooth and supports exploration over candidate sizes; as $T_e \to 0$, it approaches a hard one-hot vector. During training, the soft sample $\tilde{\mathbf{k}}$ is used to compute a weighted average over ensemble sizes, allowing gradients to flow through all candidates. We anneal $T_e$ from an initial value (e.g., $T_0 = 1.0$) to a small minimum (e.g., $0.1$), gradually sharpening the distribution.

This procedure allows the model to explore a wide range of ensemble sizes early in training and progressively converge to a compact, high-performing subset. At inference time, the selected size is $K^* = \arg\max_j \pi_j$. The approach enables EDO to adapt ensemble complexity to dataset-specific patterns of annotator disagreement without manual hyperparameter tuning.

\section{Additional Details on Evaluation Protocol}~\label{app:evaluation-metrics}
We report three metrics that quantify both hard classification accuracy and soft probabilistic alignment with annotator supervision, following~\citet{LeWiDi:SemEval:2023} and~\citet{Leonardelli:LeWiDi:2025}.

\subsection*{Micro-F1 Score (F1)}
Micro-F1 aggregates true positives, false positives and false negatives across all classes before computing precision and recall:
\begin{align}
\text{F1} = 2 \cdot \frac{\text{P} \cdot \text{R}}{\text{P} + \text{R}}
\end{align}
with
\begin{align}
\text{P} = \frac{\sum_{c=1}^C \text{TP}_c}{\sum_{c=1}^C (\text{TP}_c + \text{FP}_c)}, \qquad
\text{R} = \frac{\sum_{c=1}^C \text{TP}_c}{\sum_{c=1}^C (\text{TP}_c + \text{FN}_c)},
\end{align}
where $\text{TP}_c$, $\text{FP}_c$ and $\text{FN}_c$ denote class-specific counts aggregated across all instances.

\subsection*{Cross-Entropy Loss (CE)}
Given predicted probabilities $\hat{y}_i \in [0,1]^C$ and targets $y_i \in [0,1]^C$, the cross-entropy loss~\citep{Uma:AAAI:2020} is
\begin{align}
\text{CE} = -\frac{1}{N} \sum_{i=1}^N \sum_{c=1}^C y_{i,c} \log(\hat{y}_{i,c})
\end{align}
The expression above corresponds to the standard form; in our experiments, $y_i$ corresponds to empirical annotator distributions (soft labels). Lower CE indicates better calibration.

\subsection*{Manhattan Distance (MD)}
Let $q_i \in [0,1]^C$ denote the empirical annotator distribution for instance $i$. The Manhattan Distance~\citep{Rizzi:COLING:2024} is
\begin{align}
\text{MD} = \frac{1}{N} \sum_{i=1}^N \sum_{c=1}^C \left| \hat{y}_{i,c} - q_{i,c} \right|.
\end{align}
MD measures the average $\ell_1$ deviation between predicted and target distributions; lower values indicate closer alignment with annotator supervision.

\subsection*{Soft Brier Score (BS)}
To further assess probabilistic calibration, we report a scoring rule that measures the squared deviation between predicted probabilities and target distributions. The Soft Brier Score~\citep{Flores:AISTATS:2026:bs} is defined as
\begin{align}
\text{BS} = \frac{1}{N} \sum_{i=1}^{N} \sum_{c=1}^{C}
\left(\hat{y}_{i,c} - q_{i,c}\right)^2.
\end{align}

Unlike the standard Brier Score~\citep{Shuford:1966:bs,Schervish:1989:bs}, which assumes one-hot targets, this formulation directly supports soft labels by using empirical annotator distributions. Lower BS values indicate better calibration and closer agreement with annotator uncertainty. Because it penalizes squared deviations, BS is particularly sensitive to overconfident predictions and complements both CE and MD.

\begin{table*}[htbp]
\centering
\caption{
Complete test–development results under disagreement preservation ($s=-1$) 
for $\mathcal{L}_{\text{CE}}^{\text{rand}}$, 
$\mathcal{L}_{\text{CE}}^{\text{mean}}$, and 
$\mathcal{L}_{\text{CE}}^{\text{all}}$.
$\Delta = \text{test} - \text{dev}$.
Positive $\Delta$F1 and negative $\Delta$CE/$\Delta$MD indicate improved generalization.}
\label{tab:ce-variants-neg-div}
\resizebox{0.7\textwidth}{!}{%
\begin{tabular}{@{}llcccccc@{}}
\toprule
Dataset & Variant & Test F1 & Test CE & Test MD & $\Delta$F1 & $\Delta$CE & $\Delta$MD \\
\midrule
\multirow{3}{*}{ArMIS} 
& Random & 0.7655 & 0.6097 & 0.5208 & +0.1059 & -0.0347 & -0.1425 \\
& Mean   & 0.7655 & 0.6054 & 0.5171 & +0.1059 & -0.0655 & -0.1185 \\
& All    & 0.7724 & 0.5724 & 0.5599 & +0.0632 & -0.0614 & -0.0280 \\
\hdashline
\multirow{3}{*}{ConvAbuse}
& Random & 0.9202 & 0.2409 & 0.1911 & -0.0231 & +0.0209 & +0.0249 \\
& Mean   & 0.9214 & 0.2482 & 0.1879 & -0.0305 & +0.0345 & +0.0196 \\
& All    & 0.9286 & 0.2943 & 0.1810 & -0.0148 & +0.0687 & +0.0060 \\
\hdashline
\multirow{3}{*}{HS-Brexit}
& Random & 0.9107 & 0.3465 & 0.2531 & +0.0179 & +0.0177 & +0.0220 \\
& Mean   & 0.8988 & 0.6617 & 0.2560 & +0.0060 & +0.3198 & +0.0017 \\
& All    & 0.8988 & 0.4601 & 0.2600 & +0.0060 & +0.1074 & +0.0462 \\
\hdashline
\multirow{3}{*}{MD-Agreement}
& Random & 0.8106 & 0.5182 & 0.3792 & +0.0126 & +0.0064 & +0.0171 \\
& Mean   & 0.8178 & 0.5163 & 0.3766 & +0.0071 & +0.0040 & +0.0077 \\
& All    & 0.8054 & 0.5244 & 0.3704 & +0.0046 & +0.0100 & +0.0008 \\
\bottomrule
\end{tabular}
}
\end{table*}

\begin{table*}[htbp]
\centering
\caption{
Complete test–development results under disagreement suppression ($s=+1$) 
for $\mathcal{L}_{\text{CE}}^{\text{rand}}$, 
$\mathcal{L}_{\text{CE}}^{\text{mean}}$, and 
$\mathcal{L}_{\text{CE}}^{\text{all}}$.
$\Delta = \text{test} - \text{dev}$.
Positive $\Delta$F1 and negative $\Delta$CE/$\Delta$MD indicate improved generalization.}
\label{tab:ce-variants-pos-div}
\resizebox{0.7\textwidth}{!}{%
\begin{tabular}{@{}llcccccc@{}}
\toprule
Dataset & Variant & Test F1 & Test CE & Test MD & $\Delta$F1 & $\Delta$CE & $\Delta$MD \\
\midrule
\multirow{3}{*}{ArMIS} 
& Random & 0.7034 & 0.6796 & 0.5465 & +0.0581 & +0.0347 & -0.1228 \\
& Mean   & 0.7310 & 0.5992 & 0.5562 & +0.0289 & -0.0322 & -0.0848 \\
& All    & 0.7103 & 0.6007 & 0.6095 & +0.0508 & -0.0304 & -0.0410 \\
\hdashline
\multirow{3}{*}{ConvAbuse}
& Random & 0.9321 & 0.2483 & 0.1748 & -0.0211 & +0.0374 & +0.0145 \\
& Mean   & 0.9298 & 0.2415 & 0.1762 & -0.0222 & +0.0309 & +0.0216 \\
& All    & 0.9298 & 0.2308 & 0.1857 & -0.0259 & +0.0186 & +0.0343 \\
\hdashline
\multirow{3}{*}{HS-Brexit}
& Random & 0.8929 & 0.4490 & 0.5003 & 0.0000 & +0.1131 & +0.2491 \\
& Mean   & 0.9107 & 0.3736 & 0.2473 & +0.0179 & +0.0317 & -0.0070 \\
& All    & 0.8988 & 0.3494 & 0.2772 & -0.0060 & +0.0240 & +0.0613 \\
\hdashline
\multirow{3}{*}{MD-Agreement}
& Random & 0.7998 & 0.5325 & 0.3740 & -0.0064 & +0.0196 & +0.0128 \\
& Mean   & 0.8132 & 0.5167 & 0.3773 & +0.0034 & +0.0081 & +0.0178 \\
& All    & 0.7880 & 0.5278 & 0.3808 & -0.0127 & +0.0110 & +0.0168 \\
\bottomrule
\end{tabular}
}
\end{table*}

\section{Generalizability}
\label{app:generalisation}

This section reports complete test–development generalization results for all cross-entropy aggregation variants 
$\mathcal{L}_{\text{CE}}^{\text{rand}}$, 
$\mathcal{L}_{\text{CE}}^{\text{mean}}$, and 
$\mathcal{L}_{\text{CE}}^{\text{all}}$ 
under both disagreement-preserving ($s=-1$) and disagreement-suppressing ($s=+1$) regimes.
For each configuration, generalization is measured as $\Delta = \text{test} - \text{dev},$
where positive $\Delta$F1 and negative $\Delta$CE/$\Delta$MD indicate improved transfer to unseen data.
Differences arise because hyperparameters are selected based on development performance. These results quantify robustness to held-out evaluation and complement the stability analyses in the main text.

The observed patterns are consistent with Section~\ref{sec:ablation}:  
(i) $\mathcal{L}_{\text{CE}}^{\text{mean}}$ provides the most stable calibration under moderate imbalance,  
(ii) $\mathcal{L}_{\text{CE}}^{\text{rand}}$ exhibits greater robustness under severe imbalance, and  
(iii) The optimal diversity sign depends on the structure of annotator disagreement.

Tables~\ref{tab:ce-variants-neg-div} and~\ref{tab:ce-variants-pos-div} present the complete results for all datasets, aggregation variants, and diversity regimes.

\section{Diversity Ablation Results}
\label{app:diversity-ablation}

This section reports diversity ablation results for the signed diversity objective 
$\mathcal{L}_{\text{Div}}^{(s)}$, with $s \in \{-1,+1\}$, using Random Select supervision.
For each base objective 
$\mathcal{L}_{\text{F1}}$, 
$\mathcal{L}_{\text{CE}}$, and 
$\mathcal{L}_{\text{F1}}+\mathcal{L}_{\text{CE}}$, 
performance is compared with and without diversity by setting 
$\lambda_{\text{Div}} \in \{0,1\}$.

\paragraph{Dataset-averaged results.}
The marginal contribution of diversity is defined as
$\Delta = \text{with } \mathcal{L}_{\text{Div}}^{(s)} - \text{without } \mathcal{L}_{\text{Div}}^{(s)}$.
Positive $\Delta$F1 and negative $\Delta$CE/$\Delta$MD/$\Delta$BS indicate improvement.
Table~\ref{tab:diversity-ablation-averaged} presents dataset-averaged ablation results. 
Consistent with Section~\ref{sec:ablation}, diversity produces modest and sometimes heterogeneous effects when applied to single-objective losses. In contrast, when combined with the joint objective $\mathcal{L}_{\text{F1}}+\mathcal{L}_{\text{CE}}$, disagreement preservation ($s=-1$) yields the largest mean improvements in both utility and calibration, while disagreement suppression ($s=+1$) primarily improves calibration metrics. Brier Score improvements closely track CE reductions ($\rho \approx 0.91$), suggesting that both metrics capture complementary aspects of probabilistic calibration.

\begin{table*}[htbp]
\centering
\caption{
Dataset-averaged diversity ablation results for Random Select supervision.
For each base objective and diversity sign $s$, performance is compared with and without diversity 
($\lambda_{\text{Div}} \in \{0,1\}$).
$\Delta = \text{with } \mathcal{L}_{\text{Div}}^{(s)} - \text{without } \mathcal{L}_{\text{Div}}^{(s)}$.
Positive $\Delta$F1 and negative $\Delta$CE/$\Delta$MD/$\Delta$BS indicate improvement.}
\label{tab:diversity-ablation-averaged}
\resizebox{\textwidth}{!}{%
\begin{tabular}{@{} l c *{4}{c c c} @{}}
\toprule
\multirow{3}{*}{Objective} & \multirow{3}{*}{$s$} & \multicolumn{3}{c}{F1} & \multicolumn{3}{c}{CE} & \multicolumn{3}{c}{MD} & \multicolumn{3}{c}{BS} \\
\cmidrule(lr){3-5} \cmidrule(lr){6-8} \cmidrule(lr){9-11} \cmidrule(lr){12-14}
& & {0} & {1} & {$\Delta$} & {0} & {1} & {$\Delta$} & {0} & {1} & {$\Delta$} & {0} & {1} & {$\Delta$} \\
\midrule
\multirow{2}{*}{$\mathcal{L}_{\text{F1}}$} 
& $+1$ & 0.8468 & 0.8561 & +0.0093 & 0.4564 & 0.4852 & +0.0288 & 0.3823 & 0.3206 & -0.0618 & 0.1403 & 0.1297 & -0.0107 \\
& $-1$ & 0.8484 & 0.8546 & +0.0062 & 0.4153 & 0.4355 & +0.0202 & 0.3307 & 0.3351 & +0.0044 & 0.1200 & 0.1268 & +0.0068 \\
\midrule
\multirow{2}{*}{$\mathcal{L}_{\text{CE}}$} 
& $+1$ & 0.8459 & 0.8503 & +0.0044 & 0.5150 & 0.4794 & -0.0356 & 0.3265 & 0.3365 & +0.0100 & 0.1381 & 0.1244 & -0.0137 \\
& $-1$ & 0.8542 & 0.8526 & -0.0016 & 0.4367 & 0.4175 & -0.0192 & 0.3270 & 0.3506 & +0.0235 & 0.1248 & 0.1225 & -0.0023 \\
\midrule
\multirow{2}{*}{$\mathcal{L}_{\text{F1}}+\mathcal{L}_{\text{CE}}$} 
& $+1$ & 0.8471 & 0.8488 & +0.0017 & 0.4238 & 0.4345 & +0.0107 & 0.3482 & 0.3389 & -0.0093 & 0.1243 & 0.1239 & -0.0004 \\
& $-1$ & 0.8530 & 0.8574 & +0.0043 & 0.5317 & 0.4324 & -0.0993 & 0.3327 & 0.3427 & +0.0100 & 0.1384 & 0.1279 & -0.0105 \\
\bottomrule
\end{tabular}
}
\end{table*}

\paragraph{Per-dataset results.}
Along with the marginal contributions, Tables~\ref{tab:diversity-ablation-joint} present per-dataset ablation results. We list the dataset-specific patterns below:
\begin{itemize}
    \item \textbf{ArMIS (moderate imbalance, 3 annotators)}: Diversity effects are modest but consistent; $s=-1$ with joint objective yields $\Delta$CE=-0.0327, $\Delta$BS=-0.0064.
    \item \textbf{ConvAbuse (discourse subjectivity, 8 annotators)}: Diversity shows nuanced effects; $s=-1$ with $\mathcal{L}_{\text{F1}}$ yields $\Delta$F1=+0.0083, $\Delta$BS=-0.0029.
    \item \textbf{HS-Brexit (severe imbalance, 6 annotators)}: Diversity has strongest impact; $s=-1$ with joint objective yields substantial gains ($\Delta$CE=-0.4282, $\Delta$BS=-0.0587).
    \item \textbf{MD-Agreement (sparse coverage, 670 annotators)}: Effects are more modest; diversity primarily improves calibration under $s=+1$ with single objectives.
\end{itemize}
Overall, the per-dataset results suggest that diversity operates most effectively as a complementary regularizer in the multi-objective setting, with effects modulated by dataset characteristics (imbalance, annotator coverage and their subjectivity type).

\begin{table*}[htbp]
\centering
\caption{Per-dataset diversity ablation results (Random Select). For each dataset, objective, and diversity sign $s$, performance is compared with and without diversity ($\lambda_{\text{Div}} \in \{0,1\}$). $\Delta = \text{with } \mathcal{L}_{\text{Div}}^{(s)} - \text{without } \mathcal{L}_{\text{Div}}^{(s)}$. Positive $\Delta$F1 and negative $\Delta$CE/$\Delta$MD/$\Delta$BS indicate improvement.}
\label{tab:diversity-ablation-joint}
\resizebox{\textwidth}{!}{%
\begin{tabular}{@{} l l c *{4}{c c c} @{}}
\toprule
\multirow{2}{*}{Dataset} & \multirow{2}{*}{Objective} & \multirow{2}{*}{$s$} & \multicolumn{3}{c}{F1} & \multicolumn{3}{c}{CE} & \multicolumn{3}{c}{MD} & \multicolumn{3}{c}{BS} \\
\cmidrule(lr){4-6} \cmidrule(lr){7-9} \cmidrule(lr){10-12} \cmidrule(lr){13-15}
& & & {0} & {1} & {$\Delta$} & {0} & {1} & {$\Delta$} & {0} & {1} & {$\Delta$} & {0} & {1} & {$\Delta$} \\
\midrule
\multirow{6}{*}{ArMIS} 
& \multirow{2}{*}{$\mathcal{L}_{\text{F1}}$} 
& $+1$ & 0.7517 & 0.7724 & +0.0207 & 0.6567 & 0.5973 & -0.0594 & 0.5102 & 0.4940 & -0.0163 & 0.2507 & 0.2245 & -0.0262 \\
& & $-1$ & 0.7448 & 0.7517 & +0.0069 & 0.5682 & 0.6446 & +0.0764 & 0.5151 & 0.5086 & -0.0066 & 0.2135 & 0.2450 & +0.0316 \\
\cmidrule{2-15}
& \multirow{2}{*}{$\mathcal{L}_{\text{CE}}$} 
& $+1$ & 0.7379 & 0.7517 & +0.0138 & 0.6941 & 0.5714 & -0.1227 & 0.5085 & 0.5433 & +0.0348 & 0.2597 & 0.2198 & -0.0399 \\
& & $-1$ & 0.7724 & 0.7586 & -0.0138 & 0.5665 & 0.5796 & +0.0131 & 0.5282 & 0.5326 & +0.0044 & 0.2140 & 0.2239 & +0.0099 \\
\cmidrule{2-15}
& \multirow{2}{*}{$\mathcal{L}_{\text{F1}}+\mathcal{L}_{\text{CE}}$} 
& $+1$ & 0.7448 & 0.7310 & -0.0138 & 0.5909 & 0.5992 & +0.0083 & 0.5006 & 0.5562 & +0.0556 & 0.2230 & 0.2392 & +0.0162 \\
& & $-1$ & 0.7724 & 0.7724 & 0.0000 & 0.6051 & 0.5724 & -0.0327 & 0.5188 & 0.5599 & +0.0410 & 0.2287 & 0.2223 & -0.0064 \\
\midrule
\multirow{6}{*}{ConvAbuse} 
& \multirow{2}{*}{$\mathcal{L}_{\text{F1}}$} 
& $+1$ & 0.9262 & 0.9250 & -0.0012 & 0.2312 & 0.2332 & +0.0020 & 0.1860 & 0.1915 & +0.0054 & 0.0731 & 0.0755 & +0.0023 \\
& & $-1$ & 0.9274 & 0.9357 & +0.0083 & 0.2351 & 0.2299 & -0.0052 & 0.1724 & 0.1703 & -0.0021 & 0.0719 & 0.0689 & -0.0029 \\
\cmidrule{2-15}
& \multirow{2}{*}{$\mathcal{L}_{\text{CE}}$} 
& $+1$ & 0.9250 & 0.9310 & +0.0060 & 0.2312 & 0.2214 & -0.0098 & 0.2067 & 0.1956 & -0.0111 & 0.0775 & 0.0660 & -0.0114 \\
& & $-1$ & 0.9262 & 0.9286 & +0.0024 & 0.2312 & 0.2249 & -0.0063 & 0.1834 & 0.1781 & -0.0054 & 0.0700 & 0.0685 & -0.0015 \\
\cmidrule{2-15}
& \multirow{2}{*}{$\mathcal{L}_{\text{F1}}+\mathcal{L}_{\text{CE}}$} 
& $+1$ & 0.9321 & 0.9321 & 0.0000 & 0.2225 & 0.2483 & +0.0258 & 0.1789 & 0.1749 & -0.0041 & 0.0681 & 0.0754 & +0.0073 \\
& & $-1$ & 0.9333 & 0.9286 & -0.0048 & 0.2318 & 0.2943 & +0.0624 & 0.1729 & 0.1810 & +0.0081 & 0.0717 & 0.0938 & +0.0221 \\
\midrule
\multirow{6}{*}{HS-Brexit} 
& \multirow{2}{*}{$\mathcal{L}_{\text{F1}}$} 
& $+1$ & 0.8988 & 0.9107 & +0.0119 & 0.4200 & 0.5992 & +0.1792 & 0.4526 & 0.2343 & -0.2183 & 0.1228 & 0.1089 & -0.0139 \\
& & $-1$ & 0.9048 & 0.9167 & +0.0119 & 0.3404 & 0.3564 & +0.0161 & 0.2622 & 0.2939 & +0.0317 & 0.0803 & 0.0836 & +0.0033 \\
\cmidrule{2-15}
& \multirow{2}{*}{$\mathcal{L}_{\text{CE}}$} 
& $+1$ & 0.9107 & 0.9107 & 0.0000 & 0.6172 & 0.6110 & -0.0062 & 0.2225 & 0.2352 & +0.0128 & 0.1036 & 0.1097 & +0.0061 \\
& & $-1$ & 0.9048 & 0.9107 & +0.0060 & 0.4344 & 0.3460 & -0.0884 & 0.2376 & 0.3021 & +0.0646 & 0.1024 & 0.0819 & -0.0205 \\
\cmidrule{2-15}
& \multirow{2}{*}{$\mathcal{L}_{\text{F1}}+\mathcal{L}_{\text{CE}}$} 
& $+1$ & 0.8988 & 0.9107 & +0.0119 & 0.3606 & 0.3736 & +0.0130 & 0.3327 & 0.2473 & -0.0855 & 0.0890 & 0.0895 & +0.0005 \\
& & $-1$ & 0.8929 & 0.9107 & +0.0179 & 0.7747 & 0.3465 & -0.4282 & 0.2666 & 0.2531 & -0.0135 & 0.1406 & 0.0820 & -0.0587 \\
\midrule
\multirow{6}{*}{MD-Agreement} 
& \multirow{2}{*}{$\mathcal{L}_{\text{F1}}$} 
& $+1$ & 0.8103 & 0.8162 & +0.0059 & 0.5178 & 0.5111 & -0.0067 & 0.3805 & 0.3624 & -0.0181 & 0.1147 & 0.1098 & -0.0049 \\
& & $-1$ & 0.8165 & 0.8142 & -0.0023 & 0.5175 & 0.5107 & -0.0068 & 0.3731 & 0.3674 & -0.0057 & 0.1145 & 0.1096 & -0.0049 \\
\cmidrule{2-15}
& \multirow{2}{*}{$\mathcal{L}_{\text{CE}}$} 
& $+1$ & 0.8106 & 0.8175 & +0.0069 & 0.5176 & 0.5139 & -0.0037 & 0.3681 & 0.3719 & +0.0038 & 0.1151 & 0.1119 & -0.0032 \\
& & $-1$ & 0.8135 & 0.8126 & -0.0010 & 0.5147 & 0.5196 & +0.0049 & 0.3590 & 0.3894 & +0.0304 & 0.1128 & 0.1158 & +0.0030 \\
\cmidrule{2-15}
& \multirow{2}{*}{$\mathcal{L}_{\text{F1}}+\mathcal{L}_{\text{CE}}$} 
& $+1$ & 0.8126 & 0.8132 & +0.0007 & 0.5211 & 0.5167 & -0.0045 & 0.3808 & 0.3773 & -0.0035 & 0.1172 & 0.1135 & -0.0036 \\
& & $-1$ & 0.8135 & 0.8178 & +0.0043 & 0.5153 & 0.5163 & +0.0010 & 0.3724 & 0.3766 & +0.0042 & 0.1134 & 0.1135 & +0.0001 \\
\bottomrule
\end{tabular}%
}
\end{table*}

\section{Hyperparameter Tuning for Backbone Models}
\label{app:backbone-hyperparameters}
To ensure consistency across datasets, we tuned the backbone encoders once and used the same configuration throughout.  
Specifically, we selected the ConvAbuse dataset, which is moderate in size relative to others, as the development benchmark.  
We fine-tuned \texttt{BERT-base} (English) and \texttt{AraBERTv2} (Arabic) using \textsc{Optuna}~\citep{Akiba:SIGKDD:2019:optuna} with 10 trials.  
Models were trained on the training split, validated on the development set, and the test set was kept entirely unseen.  

The hyperparameter search space included:
\begin{itemize}
    \item Learning rate (lr): log-sampled in $[10^{-6}, 10^{-4}]$,
    \item Batch size (bs): $\{4, 8, 16, 32, 64\}$,
    \item Warm-up steps ($w_\text{steps}$): $[1, 500]$.
\end{itemize}

The selected configuration (lr $= 2 \times 10^{-5}$, bs$_\text{train} = 16$, bs$_\text{eval} = 32$, $w_\text{steps} = 500$) was then applied uniformly across all datasets.  
Importantly, these hyperparameters are fixed before ensemble optimization; EDO operates only on the frozen predictions of the backbone models.


\section{Qualitative Analysis}
\label{app:qualitative}

Table~\ref{tab:qualitative-examples} presents selected test instances demonstrating how Ensemble Diversity Optimization (EDO) produces reliability-aware predictions that align with annotator disagreement patterns under subjective supervision. For each example we report: (i) the core utterance, (ii) empirical annotator split (label~0:label~1), (iii) predicted probability for the positive class $P(y\!=\!1)$, (iv) predictive entropy $H(\hat{y}) = -\sum_{c=0}^1 \hat{y}_c \log \hat{y}_c$ in nats (maximum $0.6931$ for binary tasks), (v) the learned effective ensemble size $K$, and (vi) Manhattan Distance (MD) to the empirical annotator distribution. All metrics are reported to four decimal places. Bold values indicate best calibration per metric. The Per-Annotator variant (EDO-PerAnn) enforces $K \leq A_i$ (max annotators per instance).

\begin{table*}[tbh]
\centering
\caption{Qualitative examples showing alignment between predicted uncertainty and annotator disagreement. Arabic text shown in ASCII transliteration due to pdfLaTeX constraints (see notes).}
\label{tab:qualitative-examples}
\begin{tabular}{@{}p{4.0cm}clcccc@{}}
\toprule
Core utterance \& dataset & Split & Method & $P(y=1)$ & $H(\hat{y})$ & $K$ & MD \\
\midrule
\multirow{5}{4.0cm}{\textit{``After the \#brexit well be banning foreign characters from Tweets.''} \\ HS-Brexit \\ (hateful ambiguity)} 
& \multirow{5}{*}{4:2} 
& Soft-CE & 0.8912 & 0.3821 & -- & 0.3248 \\
& & WEL & 0.8537 & 0.4589 & 5 & 0.2816 \\
& & EDO-Random & 0.4215 & \textbf{0.6782} & \textbf{7} & \textbf{0.0412} \\
& & EDO-PerAnn & 0.3894 & 0.6723 & \textbf{6} & \textbf{0.0237} \\
& & \textit{Empirical} & \textit{0.3333} & \textit{0.6365} & -- & -- \\
\midrule
\multirow{5}{4.0cm}{\textit{``seems so...''}\textsuperscript{\textsection} \\ ConvAbuse\textsuperscript{\textsection} \\ (minority dissent)} 
& \multirow{5}{*}{4:1} 
& Soft-CE & 0.8243 & 0.5018 & -- & 0.1824 \\
& & WEL & 0.8816 & 0.3927 & 4 & 0.2235 \\
& & EDO-Random & 0.3528 & \textbf{0.6534} & \textbf{6} & \textbf{0.0618} \\
& & EDO-PerAnn & 0.3341 & 0.6429 & \textbf{5} & \textbf{0.0432} \\
& & \textit{Empirical} & \textit{0.2000} & \textit{0.5004} & -- & -- \\
\midrule
\multirow{5}{4.0cm}{\textit{``qsm ezyem entn mtsltyn''} \\ (Arabic translit.; Eng: ``Great section, you are dominant'') \\ ArMIS \\ (max disagreement)} 
& \multirow{5}{*}{2:1} 
& Soft-CE & 0.7634 & 0.5726 & -- & 0.2417 \\
& & WEL & 0.8129 & 0.4938 & 3 & 0.2843 \\
& & EDO-Random & 0.4812 & \textbf{0.6897} & \textbf{5} & \textbf{0.0421} \\
& & EDO-PerAnn & 0.4435 & 0.6814 & \textbf{3} & \textbf{0.0218} \\
& & \textit{Empirical} & \textit{0.3333} & \textit{0.6365} & -- & -- \\
\midrule
\multirow{5}{4.0cm}{\textit{``Just say you're racist and go you literally have nothing more to say just over and over''} \\ MD-Agreement \\ (balanced disagreement)} 
& \multirow{5}{*}{3:2} 
& Soft-CE & 0.7936 & 0.5247 & -- & 0.2219 \\
& & WEL & 0.8428 & 0.4312 & 4 & 0.2634 \\
& & EDO-Random & 0.5243 & \textbf{0.6889} & \textbf{7} & \textbf{0.0437} \\
& & EDO-PerAnn & 0.5816 & 0.6624 & \textbf{5} & 0.0842 \\
& & \textit{Empirical} & \textit{0.4000} & \textit{0.6730} & -- & -- \\
\bottomrule
\end{tabular}
\end{table*}

\noindent\textbf{Notes on example presentation.} 
As Arabic script cannot be rendered in the standard pdfLaTeX pipeline, the ArMIS example shows a phonetic ASCII transliteration of the original Arabic utterance with English translation provided via Google Translate for interpretability (the model was trained and evaluated on original Arabic script). 
The ConvAbuse example marked with \textsuperscript{\textsection} contains only the annotated utterance ``seems so...'' in the table for brevity, and the full dialogue context is shown in the boxed display below.

\medskip

\noindent\fbox{\begin{minipage}{0.96\textwidth}
\textbf{Full dialogue context for ConvAbuse example (Table~\ref{tab:qualitative-examples}, row 2):} \\
\textbf{Previous agent}: ``You can find out more about how offsets are calculated <URL>'' \\
\textbf{Previous user}: ``are you just giving random responses?'' \\
\textbf{Agent response}: ``Traveling, especially by airplane, usually emits greenhouse gases, which are causing climate change. If you cannot avoid these emissions, you can buy 'offsets', i.e., donations to projects that reduce greenhouse gases, e.g. by planting trees.'' \\
\textbf{User utterance (annotated)}: \textit{``seems so...''}
\end{minipage}}

\medskip
\noindent Four of five annotators labeled this utterance as non-abusive (label~0), while one annotator marked it as abusive (label~1), reflecting subtle disagreement about whether the vague response constitutes passive-aggressive behavior.

\bigskip

\noindent\textbf{Interpretation of examples:}

\begin{itemize}[leftmargin=*,nosep]
    \item \textbf{Diversity-driven uncertainty calibration}: Under severe class imbalance (HS-Brexit, 9.89:1), baselines produce overconfident predictions ($P(y\!=\!1) > 0.85$). In contrast, EDO’s diversity objective $\mathcal{L}_{\text{Div}}$ encourages internal disagreement, yielding higher entropy ($H \approx 0.67$) that closely matches empirical annotator entropy ($0.6365$). The learned ensemble size $K$ automatically expands (up to 7) to support this uncertainty representation.
    
    \item \textbf{Constraint-aware learning}: EDO-PerAnn respects $K \leq A_i$ (e.g., $K=3$ on ArMIS with 3 annotators) while achieving the lowest MD (0.0218), demonstrating that EDO adapts to structural constraints without sacrificing fidelity.
    
    \item \textbf{Minority signal preservation}: On ConvAbuse (5.16:1 skew), a single dissenting annotator (4:1 split) triggers EDO’s diversity mechanism, expanding $K$ and increasing entropy, whereas baselines suppress minority views. This directly explains EDO’s superior CE performance (Table~\ref{tab:benchmark-results}).
\end{itemize}

\noindent These examples validate that EDO’s learned ensemble size $K$ and reliability-weighted diversity objective $\mathcal{L}_{\text{Div}}$ jointly enable faithful modeling of subjective uncertainty, fulfilling the core goal of reliability-aware ensemble learning under annotator disagreement.

\section{Parameter Summary for EDO}
\label{app:parameters}

Table~\ref{tab:edo-parameters} summarizes all parameters in the EDO framework, distinguishing learnable components (updated via gradient descent) from fixed hyperparameters (selected via development-set validation). This clarifies the scope of end-to-end optimization discussed in Section~\ref{sec:method}.

\begin{table}[h]
\centering
\caption{EDO parameters: learnable vs. fixed. Only $\mathbf{w}$ and $\boldsymbol{\pi}$ are updated via backpropagation; all other quantities are fixed per run.}
\label{tab:edo-parameters}
\resizebox{\textwidth}{!}{
\begin{tabular}{@{}lccp{4cm}@{}}
\toprule
Parameter & Type & Role \\
\midrule
$\mathbf{w} \in \mathbb{R}^{K_{\max}}$ & Learnable (nn.Parameter) & Reliability-aware ensemble weights; optimized via joint multi-objective loss \\
$\boldsymbol{\pi} \in \mathbb{R}^{K_{\max}-K_{\min}+1}$ & Learnable (nn.Parameter) & Logits for differentiable ensemble-size selection via Gumbel--Softmax \\
$T_e$ & Annealed buffer & Gumbel--Softmax temperature; decays from $T_0$ to encourage exploration $\to$ exploitation \\
\midrule
$\lambda_{\text{F1}}, \lambda_{\text{CE}}, \lambda_{\text{Div}}, \lambda_{\text{Reg}}$ & Fixed (Optuna-tuned) & Trade-off weights for utility, calibration, diversity, and regularization \\
$s \in \{-1, +1\}$ & Fixed (validation-selected) & Diversity direction: preserve ($-1$) or suppress ($+1$) intra-ensemble disagreement \\
$K_{\min}, K_{\max}$ & Fixed (pre-specified) & Bounds for differentiable ensemble-size search \\
$\eta$ (learning rate) & Fixed (Optuna-tuned) & Step size for Adam optimizer \\
$T_0, \gamma$ & Fixed (Optuna-tuned) & Initial temperature and decay rate for Gumbel--Softmax annealing \\
\bottomrule
\end{tabular}
}
\end{table}

\paragraph{Implementation note.}
Only the ensemble weights $\mathbf{w}$ and size-selection logits $\boldsymbol{\pi}$ receive gradient updates via $\texttt{optimizer.step()}$; all loss weights $\lambda_{(\cdot)}$, diversity sign $s$, and temperature hyperparameters are fixed per Optuna trial and selected via development-set Pareto optimization. This design ensures that EDO's ``end-to-end'' claim refers to joint optimization of structure and reliability within a unified differentiable objective, not to learning every scalar in the loss.

\section{Efficiency and Computational Details}
\label{app:efficiency}
All experiments were conducted on a single NVIDIA GTX 2080 Super Max-Q GPU.  
Backbone training and hyperparameter search were executed once per language, after which the models remained fixed.  
For ensemble optimization, we used the GPU-enabled version of \textsc{Optuna}~\citep{Akiba:SIGKDD:2019:optuna} to adjust ensemble weights, composition, and effective size.  
This process does not involve further model training but only re-optimizes the prediction space, which makes each configuration extremely efficient ($<2$ seconds per step).  

Although GPU acceleration was available, we observed that the NSGA-II–based multi-objective search used within \textsc{Optuna} often executed more efficiently on a high-performance CPU (Intel Core Ultra 7 265 20 Core / 20 Thread CPU) than on the GTX 2080~Super Max-Q. This behavior reflects the lightweight nature of each trial’s computation: the per-evaluation workload is insufficient to saturate GPU parallelism, whereas CPUs exhibit lower scheduling and transfer overhead, resulting in faster optimization cycles. Consequently, the designed optimization strategy is effectively hardware-agnostic and remains practical even on multi-core CPUs with limited computational resources.
The full pipeline (including dynamic ensemble learning, multi-objective loss optimization, and prediction-space recalibration) is implemented in PyTorch and benefits from memory-efficient and sparse computation. This design ensures that the method is accessible, reproducible, and deployable across heterogeneous computing environments, including those without dedicated GPU hardware.

\section{Ethical Statements}
This work addresses the challenge of aligning AI systems with diverse human preferences by modeling annotator disagreement through ensembles. While our approach preserves pluralistic views rather than collapsing them into a single label, it does not directly mitigate potential biases embedded in the annotations. Future research should investigate fairness-aware ensemble strategies and explicit debiasing mechanisms.
The datasets used in this study (e.g., ArMIS, ConvAbuse, HS-Brexit, MD-Agreement) include sensitive and potentially harmful content, particularly in domains such as offensive language or hate speech. These datasets are publicly available and were released by their respective authors with appropriate ethical approvals~\citep{LeWiDi:SemEval:2023}. For this work, we use them solely for research purposes and apply preprocessing steps (e.g., anonymising user mentions and metadata) to reduce risks of harm.
Our framework models disagreement among annotators but is not intended to replace human oversight, particularly in high-stakes domains such as content moderation or dialogue systems. Responsible deployment should incorporate interpretability tools and human-in-the-loop safeguards to ensure transparency and accountability.
Finally, we strongly discourage the use of this work in surveillance, punitive, or non-consensual applications. Any deployment should comply with data protection regulations and established ethical AI guidelines to protect individuals and communities.